\documentclass[pdflatex,sn-apa,iicol]{sn-jnl}

\usepackage{graphicx}%
\usepackage{multirow}%
\usepackage{amsmath,amssymb,amsfonts}%
\usepackage{amsthm}%
\usepackage{mathrsfs}%
\usepackage[title]{appendix}%
\usepackage{xcolor}%
\usepackage{textcomp}%
\usepackage{manyfoot}%
\usepackage{booktabs}%
\usepackage{algorithm}%
\usepackage{algorithmicx}%
\usepackage{algpseudocode}%
\usepackage{listings}%

\theoremstyle{thmstyleone}%
\newtheorem{theorem}{Theorem}%
\newtheorem{lemma}{Lemma}   
\newtheorem{corollary}{Corollary}
\newtheorem{proposition}{Proposition}
\newtheorem{assumption}{Assumption}
\newtheorem{problem}{Problem}

\theoremstyle{thmstyletwo}%
\newtheorem{remark}{Remark}%

\theoremstyle{thmstylethree}%
\newtheorem{definition}{Definition}%

\usepackage{enumerate}

\usepackage{pifont}
\newcommand{\cmark}{\ding{51}}
\newcommand{\xmark}{\ding{55}}
\usepackage{adjustbox}
\usepackage{anyfontsize}

\begin{document}

\title{PnP*: A certifiable solver for Perspective-n-Point}

\author[1,2]{\fnm{Bang-Shien} \sur{Chen}}\email{chen.ban@northeastern.edu}

\author[2]{\fnm{Caitrin E.} \sur{Eaton}}\email{ceaton@draper.com}

\author[1]{\fnm{David M.} \sur{Rosen}}\email{d.rosen@northeastern.edu}

\affil[1]{\orgname{Northeastern University}, \orgaddress{\city{Boston}, \state{MA 02115}, \country{USA}}}

\affil[2]{\orgname{Draper}, \orgaddress{\city{Cambridge}, \state{MA 02139}, \country{USA}}}

\abstract{
Estimating a camera pose from 3D-2D correspondences — Perspective-n-Point — is one of the most fundamental problems in computer vision. Existing methods either solve a minimal problem or replace the objective with an easier one. Surprisingly, no scalable certifiable solver for the actual Maximum Likelihood Estimation (MLE) exists. We propose PnP*, the first practical solver that guarantees efficient computation and returns a certifiable globally optimal solution. We prove that the underlying convex relaxation is exact under moderate noise, which to our knowledge is the first stability guarantee for any relaxation for the MLE of PnP. Furthermore, we reduce the relaxation by exploiting chordal sparsity pattern and quotient ring reduction such that the relaxation is tractable and scalable. The stability and scalability results are verified numerically, as PnP* returns a provably correct camera pose at over 10 Hz.
}


\maketitle

\section{Introduction}

Perspective-n-Point (PnP) estimates a camera's pose given $N$ correspondences between 3D points and 2D measurements. It is one of the most fundamental problems in computer vision, with various applications such as structure-from-motion, visual navigation, and augmented reality.

The Maximum Likelihood Estimation (MLE) of PnP is non-convex because (\romannumeral1) the camera perspective projection is a \textit{rational function} of the pose, and (\romannumeral2) the feasible set of the pose is also non-convex. A natural approach is to minimize it locally through a nonlinear solver, e.g., Levenberg--Marquardt, but local methods require a good initialization. Common methods therefore solve an approximation of the problem in order to obtain an initial point.
\textit{Minimal solvers} fix a small number of correspondences, which yields some equation system that can be solved efficiently. However, a minimal set is insufficient to average down the noise, and leaves open which correspondences to select. \textit{Non-minimal solvers}, on the other hand, replace the MLE objective by an \textit{alternative error}\footnote{See Section~\ref{sec:related-work:pnp} for examples of alternative errors. Some literature refers to such errors as algebraic errors, and to the MLE as geometric error~\citep{HZ03}.} that can be solved efficiently or even globally, but their solutions are for a \textit{wrong} objective. 
Finally, branch-and-bound optimizes the MLE objective directly, but its computation time can grow exponentially in the worst case (see Table~\ref{table:solvers}).

\begin{table*}[t]\label{table:solvers}
\centering
\caption{Summary of existing approaches to PnP. Under moderate noise, our proposed PnP* optimizes the MLE directly and returns a provably globally optimal camera pose in polynomial time.}
\begin{adjustbox}{width=\linewidth}
\begin{tabular}{llccl}
\toprule
Approach & Objective & All points & Globally optimal & Complexity \\
\midrule
Minimal solvers   & MLE (minimal set) & \xmark &   & closed form \\
Local non-minimal solvers & alternative error   & \cmark & \xmark & linear \\
Existing certifiable solvers & alternative error   & \cmark & \cmark & polynomial \\
Branch-and-bound  & MLE  & \cmark & \cmark & exponential (worst case) \\
PnP* (ours)       & MLE  & \cmark & \cmark & polynomial \\
\bottomrule
\end{tabular}
\end{adjustbox}
\end{table*}

We propose \textit{PnP*}, the first scalable certifiable solver for the actual MLE, named after optimality and its star graph structure. We reformulate PnP as a quadratic program by introducing only the necessary auxiliary variables, and adding redundant constraints such that the problem's sparsity pattern is preserved. We prove that the resulting semidefinite relaxation is exact under moderate noise by verifying the stability conditions~\citep{CAP22}, which is to our knowledge the first stability guarantee for any relaxation of the MLE of PnP. We derive an equivalent tractable relaxation by exploiting the chordal star graph structure, and reduce the variable dimension by quotient ring reduction. In summary, PnP* is \textit{the first practical solver that guarantees efficient computation and returns a provably correct camera pose.}

The rest of the paper is organized as follows. Section~\ref{sec:related-work} reviews prior work on PnP and certifiable solvers. Section~\ref{sec:preliminaries} introduces the notation and preliminaries we use throughout. Section~\ref{sec:problem-formulation} derives the relaxation for the MLE of PnP, and Section~\ref{sec:stability} proves that the relaxation is stable. Section~\ref{sec:reduce} reformulates the relaxation into a scalable one. Section~\ref{sec:experiment} validates the stability and scalability results numerically, and finally, Section~\ref{sec:conclusion} concludes.

\section{Related work}\label{sec:related-work}

\subsection{Perspective-n-Point}\label{sec:related-work:pnp}
Perspective-n-Point (PnP) is difficult because the perspective projection is a \textit{rational} function, hence non-convex. Minimal solvers fix the number of correspondences (3 correspondences for PnP), which yields a small equation system that can be solved efficiently. Solvers in this family include~\citet{GHT03,KSS11,KR17,PN18,DYL23}. In practice, they are paired with RANSAC~\citep{FB81} and serve as hypothesis generators.

Non-minimal solvers replace the objective such that it can be solved efficiently, or even globally through certifiable solvers.
In particular, \citet{LMF09} linearize an alternative \textit{weighted reprojection error}\footnote{The weighted reprojection error multiplies each residual by the denominator of the rational term such that the new (weighted) residual is linear.}, \citet{TL20} minimize it by sequential quadratic programs, and \citet{ZKS13} find the critical points of the alternative error using Gr\"obner basis~\citep{KBP08}.
\citet{HR11} and \citet{KLS14} find the critical points of another alternative \textit{point-line error}\footnote{The point-line error instead measures the distance from each 3D point to the ray of its measurement, again making the new residual linear.}.
\citet{SD20} and \citet{AGD23} provide certifiable solvers for the weighted reprojection error and point-line error, respectively, and empirically return the corresponding global solution.
Nevertheless, those optimal solutions are optimal with respect to some alternative error instead of the actual MLE objective. 

Finally, branch-and-bound methods~\citep{OKO06} optimize the MLE objective directly, but the computation time can grow exponentially in the worst case. Subsequent work again modifies the objective by adopting the $L_\infty$ norm~\citep{EK08,HK09} to speed up computation.

\subsection{Certifiable solvers}
Certifiable solvers offer another approach to solving non-convex problems.
Many estimation problems are naturally posed as non-convex quadratically constrained quadratic programs (QCQPs). Such problems admit convex semidefinite program (SDP) relaxations~\citep{LMS10} that can be solved globally in polynomial time. If the relaxation is exact, they also recover the global minimizer of the non-convex QCQP.
This framework has been well studied in robotics applications, such as point cloud registration~\citep{BG17}, rotation averaging~\citep{EOK18}, pose graph optimization~\citep{RCB19}, and hand-eye calibration~\citep{WKC26}. Furthermore, scalability has been addressed by exploiting chordal sparsity~\citep{VA15} or leveraging low-rank SDP solvers~\citep{B15,R20}.

However, a relaxation may not always be exact. In such cases, a common solution is to consider higher-order moment relaxation~\citep{L09} or SOS relaxation~\citep{BPT12} for tighter relaxations, but their computation also increases exponentially. 
In computer vision, \citet{KH07} applied moment relaxations to triangulation, camera resectioning, homography and epipolar geometry, and such relaxations have since been developed for various problems; we refer to the recent survey by \citet{ZYL26}. 
To the best of our knowledge, certifiable solvers have reached PnP only through alternative errors (see Section~\ref{sec:related-work:pnp}).


In this work, we build a scalable certifiable solver for the MLE of PnP.
The closest work to ours~\citep{C21} proposed a stable relaxation for camera resectioning, which estimates a general projective camera by minimizing the same objective, but without the rotation constraint. 
Our relaxation (see Section~\ref{sec:problem-formulation}) builds on theirs by adding the rotation constraint, which changes the feasible set and so stability no longer follows. We verify the stability conditions and show that our relaxation is stable (see Section~\ref{sec:stability}).
Moreover, developing a scalable solver for their relaxation was left as an open problem in \citet{C21}. We exploit our problem's chordal sparsity and reduce the problem's dimension by quotient ring reduction (see Section~\ref{sec:reduce}), and show that our method is more efficient than existing global methods for PnP.

\section{Preliminaries}\label{sec:preliminaries}

\subsection{Notation}
We write vectors in lowercase and matrices in uppercase. We set $[n]\triangleq\{1,\dots,n\}$, write $I_n$ for the $n\times n$ identity, $0_{m\times n}$ for the $m\times n$ zero matrix, and $e_i$ for the $i$-th standard basis of $\mathbb{R}^n$. $\mathbb S^n$ denotes the set of real symmetric $n\times n$ matrices and $\mathbb S^n_+$ its subset of positive semidefinite matrices. 
Let $\mathrm{SO}(3)=\{R\in\mathbb{R}^{3\times3}\mid R^\top R=I_3,\ \det(R)=1\}$ be the rotation group. We use the column-major vectorization $\mathrm{vec}(A)$ that stacks the columns of $A$, and the symmetric part $\mathrm{sym}(A)=(A+A^\top)/2$. For $k$ dividing $n$, we denote by $\mathcal{BS}^n_k$ the set of $n\times n$ matrices partitioned into $k\times k$ blocks, in which every block is symmetric.

For a scalar function $f:\mathbb R^n\to\mathbb R$, $\nabla f$ and $\nabla^2f$ denote its gradient and Hessian. For a vector-valued function $f:\mathbb R^n\to\mathbb R^m$, $\mathbf J_f\in\mathbb R^{m\times n}$ is its Jacobian. For a subvector $\alpha$, $\nabla_\alpha$, $\nabla^2_{\alpha\alpha}$, and $\mathbf J_{f,\alpha}$ denote the partial gradient, Hessian, and Jacobian with respect to $\alpha$. $T_x(\mathcal M)$ is the tangent space of a manifold $\mathcal M$ at $x$, and $dF_x$ the differential of a smooth map $F$ at $x$. Finally, $\mathbb R[x]$ denotes the ring of real polynomials in $x$, and $\mathbb R[x]/I$ denotes its quotient by an ideal $I$.

\subsection{Shor's relaxation}\label{sec:shor}
We first introduce the main idea behind certifiable solvers.
Consider a quadratically constrained quadratic program (QCQP):
\begin{equation}\label{eq:qcqp:inhomo}
\begin{aligned}
p^*_\text{QCQP}=\min_{\tilde x}\ &\tilde x^\top P_0\tilde x+2q_0^\top\tilde x+r_0\\
\text{s.t.}\ &\tilde x^\top P_j\tilde x+2q_j^\top\tilde x+r_j=0,\quad j\in [\ell]
\end{aligned}\tag{QCQP}
\end{equation}
where $\tilde x$ is some \textit{optimization variable}. Since
\begin{equation}\label{eq:inhomo2homo}
\tilde{x}^\top P\tilde{x}+2q^\top \tilde{x}+r=\begin{pmatrix}
    1\\\tilde{x}
\end{pmatrix}^\top\begin{pmatrix}
    r&q^\top\\q&P
\end{pmatrix}\begin{pmatrix}
    1\\\tilde{x}
\end{pmatrix},
\end{equation}
any quadratic function $\tilde{x}^\top P\tilde{x}+2q^\top \tilde{x}+r$ can be written in a homogeneous form $x^\top Cx$ where $x=(1,\tilde x)$.
By introducing a \textit{homogeneous variable} $h$ and adding the quadratic constraint $h^2=1$, we can rewrite \eqref{eq:qcqp:inhomo} as a homogeneous QCQP:
\begin{equation}\label{eq:qcqp:homo}
\begin{aligned}
p^*_\text{QCQP}=\min_{x=(h,\tilde x)}\ &x^\top C_0x\\
\text{s.t.}\quad &x^\top C_jx=0,\quad j\in [\ell]\\
&h^2=1
\end{aligned}\tag{Homo-QCQP}
\end{equation}

A well-known relaxation for \eqref{eq:qcqp:homo} is the Shor's relaxation. Since $x^\top Cx=C\bullet xx^\top$ and the matrix $X'\triangleq xx^\top$ satisfies $X'\succeq0$ and $\mathrm{rank}(X')=1$, by dropping the rank condition, we have the semidefinite program (SDP):
\begin{equation}\label{eq:sdp:shor}
\begin{aligned}
p^*_\text{SDP}=\min_{X\succeq0}\ &C_0\bullet X\\
\text{s.t.}\ &C_j\bullet X=0,\quad j\in[\ell]\\
&E_{00}\bullet X=1
\end{aligned}\tag{Shor-SDP}
\end{equation}
where $E_{00}$ is the elementary matrix with $1$ in the top-left entry and zeros elsewhere. Since the rank condition is dropped, an optimal solution $X^*$ of \eqref{eq:sdp:shor} may have a rank greater than 1. We immediately obtain a certificate: if $\text{rank}(X^*)=1$, then we can factorize $X^*=x^*(x^*)^\top$ where $x^*$ is an optimal solution of \eqref{eq:qcqp:homo}.

\begin{definition}[Tightness and Exactness]
We say a relaxation is \textit{tight} if the optimality gap is zero.
That is, if $p^*_\text{SDP}=p^*_\text{QCQP}$, then \eqref{eq:sdp:shor} relaxation is tight.
Furthermore, we say the relaxation is \textit{exact} if it is tight and has a unique rank 1 solution.
\end{definition}

\subsection{Higher-order relaxations}\label{sec:higher-order}
In practice, Shor's relaxation may not always be tight. 
In general, a tighter relaxation can be obtained by reformulating the QCQP with respect to a \textit{lifted basis} that contains higher-order monomials, adding \textit{redundant constraints}\footnote{A constraint is said to be redundant if it is satisfied by every feasible point of the QCQP. Adding the constraint leaves the QCQP unchanged, while the corresponding constraint in the SDP relaxation can shrink its feasible set, potentially tightening the relaxation.} that are up to degree 2 with respect to the lifted variable, and applying Shor's relaxation to the lifted problem. In particular, \eqref{eq:qcqp:homo} is the QCQP with respect to the simplest basis $[1,\tilde x]$, which consists of monomials up to degree 1 of the optimization variable $\tilde x$.

\begin{remark}
We distinguish a basis from its variable: a basis is a list of monomials written in square brackets, e.g., $[1,\tilde x]$; the corresponding variable is written in parentheses, e.g., $(h,\tilde x)$.
\end{remark}

We illustrate both steps with the case that arises in Section~\ref{sec:problem-formulation}. Consider the basis $[1, z, v]$ and a quadratic (degree 2) constraint
\begin{equation}\label{eq:bilinear}
  v^\top C z = c^\top(v \otimes z)=0,
\end{equation}
where $c=\mathrm{vec}(C^\top)$.
We can lift the basis from $[1,z,v]$ to $[1,z,v\otimes z]$ with the lifted variable $(h,z,y)$, where
\begin{equation}\label{eq:lift}
  y \triangleq v \otimes z.
\end{equation}
Note that the entries of $v\otimes z$ are higher-order monomials with respect to the monomials $v$ and $z$.
The constraint \eqref{eq:bilinear} now reads $c^\top y = 0$, which is linear (degree 1) in the lifted variable. 
Therefore, multiplying it by the degree 1 monomials of the lifted basis yields degree 2 redundant constraints
\begin{equation}\label{eq:redundant}
  (c^\top y)\begin{pmatrix}
      h&z^\top&y^\top
  \end{pmatrix}=0.
\end{equation}

To summarize, a higher-order relaxation\footnote{Moment relaxation~\citep{L09} gives a structured way of lifting the basis and constructing redundant constraints. In particular, Shor's relaxation is the level-1 moment relaxation.} is specified by two choices: the lifted basis and the redundant constraints. A larger basis with more redundant constraints generally yields a tighter SDP relaxation, but at a higher computational cost. In Section~\ref{sec:problem-formulation}, we make these choices for our problem.

\subsection{Problem equivalence}\label{sec:equivalence}
In optimization literature~\citep{BV04}, two optimization problems are considered equivalent if a solution to one can be readily found from a solution to the other. Because our problem reformulation hinges on proving the uniqueness of the optimal solution for equivalent problems, we require a stronger notion of equivalence.
A subtlety, however, arises from homogenization. 

Recall that the objective and constraints of \eqref{eq:qcqp:homo} are homogeneous quadratic, hence the feasible set is invariant under a sign flip, and both $x$ and $-x$ attain the same objective value. 
That is, the solutions of \eqref{eq:qcqp:homo} naturally come in pairs. We define the following identification.

\begin{definition}[Sign-equivalence]\label{def:sign-equivalence}
Two feasible points $x$ and $x'$ of \eqref{eq:qcqp:homo} are \textit{equivalent} if $x' = \pm x$. The equivalence class of $x$ is $[x] \triangleq \{\pm x\}$.
\end{definition}

This is consistent with \eqref{eq:sdp:shor}, where a rank 1 solution recovers one equivalence class. For the other problems we consider that are not homogeneous QCQPs, every solution forms its own class. To this end, we give our notion of problem equivalence.

\begin{definition}[Problem equivalence]\label{def:prob-equivalence}
Two problems are \textit{equivalent} if there exists a bijective mapping between the equivalence classes of their optimal sets.
\end{definition}

Homogenizing a QCQP preserves equivalence in this sense (see Lemma~\ref{lem:homogenization}).

\section{Perspective-n-Point and SDP relaxation}\label{sec:problem-formulation}
Given $N$ known world points $\{p_i\}_{i=1}^N\subset\mathbb{R}^3$ and 2D measurements $\{\theta_i\}_{i=1}^N\subset\mathbb{R}^2$, the Perspective-n-Point (PnP) problem estimates the camera pose that consists of a rotation $R\in\text{SO}(3)$ and a translation $t\in\mathbb{R}^3$. Let
$$
K=\begin{pmatrix}
    f_x&s&c_x\\
    0&f_y&c_y\\
    0&0&1
\end{pmatrix}
$$
be a \textit{known} invertible camera intrinsic matrix with non-zero focal lengths $f_x,f_y$.
The pinhole perspective projection $\Pi:\{x\in\mathbb{R}^3\mid e_3^\top x\neq0\}\to\mathbb{R}^2$ is defined by
$$
\Pi(x)\triangleq\frac{\begin{bmatrix} I_2 & 0\end{bmatrix}Kx}{e_3^\top x}.
$$
Each world point, expressed in the camera frame as $Rp_i+t$, is projected by $\Pi(Rp_i+t)$.

Throughout the paper, we parameterize the camera pose by $z=(\text{vec}(R),t)\in\mathbb{R}^{12}$, under which the rigid transformation is linear:
$$
Rp_i+t=P_iz,\quad P_i\triangleq(p_i,1)^\top\otimes I_3\in\mathbb{R}^{3\times12}.
$$
Accordingly, we define the projection $\pi:\{z\in\mathbb{R}^{12}\mid e_3^\top P_iz\neq0,\ i\in[N]\}\to\mathbb{R}^{2N}$ by
\begin{equation}\label{eq:projection}
  \pi_i(z)\triangleq\Pi(P_iz)=\frac{\begin{bmatrix}I_2&0\end{bmatrix}KP_iz}{e_3^\top P_iz}.
\end{equation}

Let $z_\text{gt}\triangleq(\text{vec}(R_\text{gt}),t_\text{gt})$ denote the ground truth camera pose with $e_3^\top P_iz_\text{gt}\neq0$ for all $i\in[N]$ so that the projections are defined. We refer to
\begin{equation}\label{eq:noiseless}
  \bar\theta\triangleq\pi(z_\text{gt})\in\mathbb{R}^{2N}
\end{equation}
as the \textit{noiseless measurements}, and measurements $\theta\in\mathbb{R}^{2N}$ obtained by the noise model
\begin{equation}\label{eq:noise}
\theta_i=\bar\theta_i+\epsilon_i,\quad\epsilon_i\sim\mathcal N(0,\sigma^2 I_2).
\end{equation}
The \textit{Maximum Likelihood Estimation (MLE)} of PnP is as follows.
\begin{problem}[MLE for PnP]\label{prob:mle}
$$
\min_{z=(\mathrm{vec}(R),t)\in\mathbb{R}^{12}}\sum_{i=1}^N\|\theta_i-\pi_i(z)\|^2\quad
\mathrm{s.t.}\quad R\in\mathrm{SO}(3)
$$
\end{problem}

\begin{assumption}\label{assump:nondegenerate}
We assume that $N \geq 6$, and the world points together with the camera center at the ground truth pose do not lie in any critical sets for camera resectioning~\citep[Result 22.5]{HZ03}. 
\end{assumption}

\begin{proposition}\label{prop:mle:unique}
Under Assumption~\ref{assump:nondegenerate}, if $\theta=\bar\theta$, then $z_\mathrm{gt}$ is the unique global minimizer of Problem~\ref{prob:mle}.
\end{proposition}

The proof is given in Appendix~\ref{appx:proof:mle:unique}. Note that uniqueness is required for the exactness of the SDP relaxation.

\begin{remark}
Assumption~\ref{assump:nondegenerate} is stronger than what uniqueness alone requires since a camera resectioning solution need not be a PnP solution. We adopt the critical sets of camera resectioning because our stability results in Section~\ref{sec:stability} build on the linear system in the proof of Proposition~\ref{prop:mle:unique}.
When $N<6$, Problem~\ref{prob:mle} can have multiple solutions~\citep{FB81}, in which case the relaxation is not exact. This is also consistent with the experiments in Section~\ref{sec:exp:scalability}.
\end{remark}

\subsection{QCQP reformulations}\label{sec:qcqp-reformulations}
Problem~\ref{prob:mle} is not a QCQP for two reasons: (\romannumeral1) the objective is not quadratic because $\pi_i$ is a rational function, and (\romannumeral2) the constraint $R\in\mathrm{SO}(3)$ is also not quadratic. 
For the rotation constraint, \citet{TRC15} showed that $R^\top R=I_3$ and $\det(R)=1$ together are equivalent to 15 quadratic constraints. Alternatively, it can be rewritten as 20 linearly independent quadratic constraints~\citep{BG17,BKG18} by adding redundant constraints via $RR^\top=I_3$.
We adopt the latter formulation and write it as homogeneous quadratic constraints 
$$
r_j(h,z)=0,\quad j\in[20],
$$
where $h$ is the homogeneous variable; their explicit forms are given in Appendix~\ref{appx:rot}.

For the rational objective, we introduce auxiliary variables $v\in\mathbb{R}^{2N}$ where $v_i\triangleq\theta_i-\pi_i(z)$ such that the objective becomes $\|v\|^2$, which is now quadratic. This moves the rational term into the constraints
\begin{equation}\label{eq:aux:v}
v_i=\theta_i-\pi_i(z)=\theta_i-\frac{\begin{bmatrix}I_2&0\end{bmatrix}KP_iz}{e_3^\top P_iz}
\end{equation}
for $i\in [N]$. Multiplying by the denominator yields $m_{\theta,i}(z,v)=0$ where
\begin{equation}\label{eq:model}
m_{\theta,i}(z,v)\triangleq\theta_i(e_3^\top P_iz)-\begin{bmatrix}I_2&0\end{bmatrix}KP_iz-v_i(e_3^\top P_iz),
\end{equation}
which is also quadratic. Altogether, we rewrite Problem~\ref{prob:mle} as the following QCQP with respect to the basis $[1,z,v]$ and variable $w\triangleq(h,z,v)\in\mathbb{R}^{13+2N}$.

\begin{problem}[Naive QCQP for PnP]\label{prob:naive}
\begin{align*}
\min_{w=(h,z,v)\in\mathbb{R}^{13+2N}}\ &\|v\|^2\\
\mathrm{s.t.}\qquad\ \ \ & m_{\theta,i}(z,v)=0,\quad i\in[N]\\
&r_j(h,z)=0,\quad j\in[20]\\
&h^2=1
\end{align*}
\end{problem}

We call the first and second sets of constraints the \textit{model constraints} and \textit{rotation constraints}, and the last one the \textit{homogeneous constraint}. 

\begin{theorem}\label{thm:naive}
Under Assumption~\ref{assump:nondegenerate}, given $\theta$ in a neighborhood of $\bar\theta$, Problem~\ref{prob:naive} is equivalent to Problem~\ref{prob:mle}.
\end{theorem}

\begin{corollary}\label{cor:naive}
Under Assumption~\ref{assump:nondegenerate}, if $\theta=\bar\theta$, then Problem~\ref{prob:naive} has the unique equivalence class of global minimizer $[\bar w]$ with $\bar w=(1,z_\mathrm{gt},0)$.
\end{corollary}

The proofs of Theorem~\ref{thm:naive} and Corollary~\ref{cor:naive} are given in Appendix~\ref{appx:proof:naive} and Appendix~\ref{appx:proof:naive:unique}, respectively.

\begin{remark}\label{remark:depth}
Problem~\ref{prob:mle} implicitly requires $e_3^\top P_iz\neq0$ for all $i\in[N]$, since $\pi$ is defined only there, but Problem~\ref{prob:naive} can admit feasible points with $e_3^\top P_iz=0$. This implies $\begin{bmatrix}I_2&0\end{bmatrix}KP_iz=0$ by \eqref{eq:model}, i.e., the world point coincides with the estimated camera center. Assumption~\ref{assump:nondegenerate} excludes such feasible points from being optimal for $\theta$ in a neighborhood of $\bar\theta$ (see Lemma~\ref{lemma:feasible-set:naive:restricted}).
\end{remark}

Problem~\ref{prob:naive} is now a QCQP and one can apply SDP relaxations. However, its \eqref{eq:sdp:shor} relaxation is \textit{non-informative}~\citep[Lemma 3.1]{C21}, meaning the relaxation will always have an optimal value of zero regardless of the measurement. 
This is because the model constraints are bilinear in $v$ and $z$, while the objective involves only $\|v\|^2$. In the SDP relaxation, every product of entries of $w$ becomes an independent variable, so the model constraint restricts only the product of $v$ and $z$, while the objective involves the product of $v$ with itself. The relaxation is free to decouple them and set the objective value to zero for any $\theta$ (see Proposition~\ref{prop:non-informative}).

To address this, we introduce new variables $y\triangleq v\otimes z$ as the product of $v$ and $z$, and replace the objective $\|v\|^2$ by $\|y\|^2$.
Following Section~\ref{sec:higher-order}, we lift the basis $[1,z,v]$ to $[1,z,v\otimes z]$ as in~\eqref{eq:lift}, with variable $x\triangleq(h,z,y)\in\mathbb{R}^{13+24N}$.
By \eqref{eq:bilinear}, the model constraints~\eqref{eq:model} become linear (degree 1) in $x$:
\begin{equation}\label{eq:model:lifted}
g_{\theta,i}(x)\triangleq\theta_i(e_3^\top P_iz)-\begin{bmatrix}I_2&0\end{bmatrix}KP_iz-(I_2\otimes(e_3^\top P_i))y_i.
\end{equation}
As in~\eqref{eq:redundant}, multiplying $g_{\theta,i}$ by (degree 1) monomials of the lifted basis yields quadratic (degree 2) redundant constraints. In particular, we take those with monomials of point $i$, collected in the subvariable 
\begin{equation}\label{eq:subvariable}
    x^{(i)}\triangleq(h,z,y_i)\in\mathbb{R}^{37}.
\end{equation}
This choice preserves a sparsity pattern which we will exploit in Section~\ref{sec:reduce}. Finally, we replace the homogeneous constraint $h^2=1$ by  $\|z\|^2=1$ such that the new objective $\|y\|^2=\|v\otimes z\|^2=\|v\|^2\|z\|^2=\|v\|^2$ reduces to the objective of Problem~\ref{prob:naive}. Since $(z,y_i)(z,y_i)^\top\in\mathcal{BS}^{36}_{12}$ and $z\neq0$ ensures $y_i=v_i\otimes z$ for some $v_i\in\mathbb{R}^2$~\citep[Section~3.3]{C21}, we have the following QCQP.

\begin{problem}[Lifted QCQP for PnP]\label{prob:lifted}
\begin{align*}
\min_{x=(h,z,y)\in\mathbb{R}^{13+24N}}\ &\|y\|^2\\
\mathrm{s.t.}\quad\quad\quad &g_{\theta,i}(x)(x^{(i)})^\top=0,\quad i\in[N]\\
&r_j(h,z)=0,\quad j\in[20]\\
&\|z\|^2=1\\
&u_iu_i^\top\in\mathcal{BS}^{36}_{12},\quad i\in[N]
\end{align*} 
where $u_i=(z,y_i)$.
\end{problem}

We refer to the first set of constraints as the \textit{lifted model constraints}, the third one as the \textit{normalization constraint}, and the last one as the \textit{block symmetry constraint}.

\begin{theorem}\label{thm:lifted}
Under Assumption~\ref{assump:nondegenerate}, given $\theta$ in a neighborhood of $\bar\theta$, Problem~\ref{prob:lifted} is equivalent to Problem~\ref{prob:naive}.
\end{theorem}

\begin{corollary}\label{cor:lifted}
Under Assumption~\ref{assump:nondegenerate}, if $\theta=\bar\theta$, then Problem~\ref{prob:lifted} has the unique equivalence class of global minimizer $[\bar x]$ with $\bar x=(\bar h,\bar z,0)$, and $\bar z/\bar h$ is the unique global minimizer of Problem~\ref{prob:mle}.
\end{corollary}

The proofs of Theorem~\ref{thm:lifted} and Corollary~\ref{cor:lifted} are again given in Appendix~\ref{appx:proof:lifted} and Appendix~\ref{appx:proof:lifted:unique}, respectively. 

\begin{remark}\label{remark:h}
Problem~\ref{prob:naive} requires $h\neq0$, but Problem~\ref{prob:lifted} can admit feasible points with $h=0$. Then by the rotation and normalization constraints, $z=(0,t)$ with $\|t\|=1$, which implies all projections of the world points by $z$ are the same pixel point by \eqref{eq:model:lifted}. Assumption~\ref{assump:nondegenerate} excludes such feasible points from being optimal for $\theta$ in a neighborhood of $\bar\theta$ (see Lemma~\ref{lemma:feasible-set:lifted:restricted}).
\end{remark}

We next derive the \eqref{eq:sdp:shor} relaxation of Problem~\ref{prob:lifted}, and show that the relaxation is exact in a neighborhood of $\bar\theta$ in Section~\ref{sec:stability}.

\subsection{Semidefinite relaxation}\label{sec:sdp}
Since each lifted model constraint $g_{\theta,i}(x)(x^{(i)})^\top=0$ gives $2\times 37$ scalar constraints, we write them as $x^\top C_{g_{\theta,i},k}\,x=0$ for $k\in[74]$. Similarly, we write the rotation constraints as $x^\top C_{r_j}x=0$ for $j\in[20]$. 
Following Section~\ref{sec:shor}, the \eqref{eq:sdp:shor}
relaxation of Problem~\ref{prob:lifted} is
\begin{equation}\label{eq:sdp}
\begin{aligned}
\min_{X\in\mathbb{S}_+^{13+24N}}\ &\mathrm{trace}(X_{yy})\\
\text{s.t.}\quad\ &C_{g_{\theta,i},k}\bullet X=0,\quad k\in[74],\ i\in[N]\\
&C_{r_j}\bullet X=0,\quad j\in[20]\\
&\mathrm{trace}(X_{zz})=1\\
&X_{u_iu_i}\in\mathcal{BS}^{36}_{12},\quad i\in[N]
\end{aligned}\tag{SDP}
\end{equation}
where for a subvector $\alpha$ of $x$, $X_{\alpha\alpha}$ denotes the
principal submatrix of $X$ on the rows and columns indexed by $\alpha$.

\section{Stability of the SDP relaxation}\label{sec:stability}

\begin{definition}[SDP Stability]\label{def:stable}
We say \eqref{eq:sdp} is \textit{stable} near $\bar\theta$ if it is exact for every $\theta$ in a neighborhood of $\bar\theta$.
\end{definition}

We now state our main theoretical results.

\begin{theorem}\label{thm:stable}
Under Assumption~\ref{assump:nondegenerate}, \eqref{eq:sdp} is stable near
$\bar\theta$.
\end{theorem}

The complete proof is given in Appendix~\ref{secB}, where we verify the SDP stability conditions (see Lemma~\ref{lem:Cifuentes}).
Theorem~\ref{thm:stable} guarantees that given $\theta$ in a neighborhood of $\bar\theta$, the relaxation will be exact. Together with the results in Section~\ref{sec:problem-formulation}, the solution of \eqref{eq:sdp} recovers the unique global minimizer of Problem~\ref{prob:mle}.

\begin{corollary}\label{cor:stable}
Under Assumption~\ref{assump:nondegenerate}, given $\theta$ in a neighborhood of $\bar\theta$, \eqref{eq:sdp} has the unique optimal solution $x^*(x^*)^\top$ with $x^*=(h^*,z^*,y^*)$, and $z^*/h^*$ is the unique global minimizer of Problem~\ref{prob:mle}.
\end{corollary}

The proof is given in Appendix~\ref{appx:cor:stable}. Corollary~\ref{cor:stable} is what our solver relies on in practice: under moderate noise, we can solve the relaxation and recover the globally optimal camera pose.

\section{Reducing the SDP relaxation}\label{sec:reduce}

\subsection{Chordal sparsity}\label{sec:reduce:quotient} 
While \eqref{eq:sdp} is stable, it scales poorly with the number of points $N$. Notice, however, that the objective of Problem~\ref{prob:lifted} decomposes as $\|y\|^2=\sum_{i=1}^N\|y_i\|^2$, and that each constraint involves only a subvariable $x^{(i)}=(h,z,y_i)$ in \eqref{eq:subvariable}. That is, the problem inherits chordal sparsity~\citep{GJS84,VA15}. 

The variable dependency graph forms a star graph (see Figure~\ref{fig:star}), where $\mathbf h=(h,z)$ is the \textit{hub variable} and $y_i$ is the $i$-th \textit{spokes variable}. 
The maximal cliques correspond to $x^{(i)}\in\mathbb{R}^{37}$. Let $X^{(i)}\triangleq X_{x^{(i)}x^{(i)}}\in\mathbb S^{37}$. Instead of tracking a large variable $X\succeq0$, we track the $N$ smaller variables $X^{(i)}\succeq0$ such that $X$ is positive semidefinite completable by Grone's theorem (see Lemma~\ref{lemma:grone}). This reduces to the following SDP.
\begin{equation}\label{eq:chordal:sdp}
\begin{aligned}
\min_{X^{(i)}\in\mathbb{S}_+^{37}}\ &\sum_{i=1}^N\mathrm{trace}(X^{(i)}_{y_iy_i})\\
\text{s.t.}\quad &\tilde C_{g_{\theta,i},k}\bullet X^{(i)}=0,\quad k\in[74],\ i\in[N]\\
&\tilde C_{r_j}\bullet X^{(1)}_{\mathbf{hh}}=0,\quad j\in[20]\\
&\mathrm{trace}(X_{zz}^{(1)})=1\\
&X^{(i)}_{u_iu_i}\in\mathcal{BS}^{36}_{12},\quad i\in[N]\\
&X^{(i)}_\mathbf{hh}=X^{(1)}_\mathbf{hh},\quad i=2,\dots,N
\end{aligned}\tag{Chordal-SDP}
\end{equation}
where $\tilde C_{g_{\theta,i},k}\in\mathbb S^{37}$ and $\tilde C_{r_j}\in\mathbb S^{13}$ are the principal submatrices of $C_{g_{\theta,i},k}$ and $C_{r_j}$ in~\eqref{eq:sdp}, respectively, which contain all their nonzero entries.

\begin{figure}[t]
    \centering
    \includegraphics[width=1\linewidth]{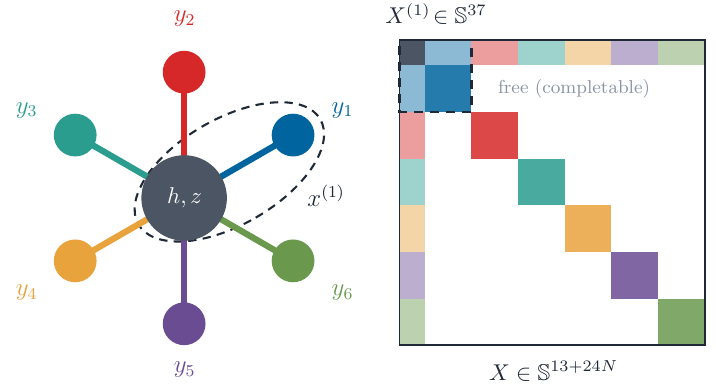}
    \caption{Chordal star graph for $N=6$. Left: the variable dependency graph with hub $(h,z)$ and spokes $y_i$, each maximal clique consists of the hub with a spoke $x^{(i)}=(h,z,y_i)$. Right: the corresponding variable, colored by the clique blocks.}
    \label{fig:star}
\end{figure}

\begin{theorem}\label{thm:chordal}
Under Assumption~\ref{assump:nondegenerate}, given $\theta$ in a neighborhood of $\bar\theta$, \eqref{eq:chordal:sdp} is equivalent to \eqref{eq:sdp}.
\end{theorem}

The proof is given in Appendix~\ref{appx:proof:chordal}.
Instead of solving \eqref{eq:sdp}, we now solve \eqref{eq:chordal:sdp}. The variable dimension of \eqref{eq:sdp} scales $\mathcal{O}(N^2)$, while the variable dimension of \eqref{eq:chordal:sdp} scales $N\cdot\mathcal{O}(1)=\mathcal{O}(N)$. This drastically reduces the computation time, as we show in Section~\ref{sec:experiment}.

\begin{remark}
Assumption~\ref{assump:nondegenerate} is needed \textit{not} for the chordal decomposition, but for \eqref{eq:sdp} to have a unique optimal solution such that the bijection in our notion of problem equivalence (Definition~\ref{def:prob-equivalence}) exists.
For example, if \eqref{eq:sdp} has two optimal solutions $X\neq X'$ with the same clique blocks, then both correspond to the same optimal solution of \eqref{eq:chordal:sdp}, and so the optimal sets are not in bijection.
\end{remark}

\subsection{Quotient ring reduction}\label{sec:reduce:qrr}
Instead of using a basis of the polynomial ring $\mathbb R[x]$, one can select a basis of the quotient ring $\mathbb R[x]/\mathcal I$, where $\mathcal I$ is the ideal of the feasible set of Problem~\ref{prob:lifted}~\citep{BPT12,CP17}. This removes the monomials that are redundant on the feasible set, but requires computing a Gr\"obner basis of $\mathcal I$, which is effective only when a Gr\"obner basis can be obtained easily. 

We therefore restrict our focus to the \textit{linear} polynomials of $\mathcal I$ that involve only $x^{(i)}$. Denote by $\mathcal I_i$ the ideal they generate, then finding the quotient ring $\mathbb{R}[x^{(i)}]/\mathcal{I}_i$ reduces to a linear transformation.
We next construct three linear polynomials of $\mathcal I_i$, two from the model constraints \eqref{eq:model:lifted} and another redundant constraint, and show that they reduce each block $X^{(i)}$ by 3 dimensions\footnote{We observe this dimension reduction numerically by applying the sampling-based methods \citep{CP17} to~\eqref{eq:chordal:sdp}.}.

Recall that the constraint $g_{\theta,i}(x)=0$ in \eqref{eq:model:lifted} is \textit{linear} in $x^{(i)}$. Let
$$
A_i\triangleq\theta_ie_3^\top P_i-\begin{bmatrix}I_2&0\end{bmatrix}KP_i\in\mathbb R^{2\times12},
$$
which is the $i$-th block of $A(\theta)$ in~\eqref{eq:affine},  we can write $g_{\theta,i}(x)=0$ as
\begin{align}
\tilde{Q}_ix^{(i)}&=0,\label{eq:qrr:model} \\
\tilde{Q}_i&=\begin{bmatrix}0_{2\times1}&A_i&-I_2\otimes (e_3^\top P_i)\end{bmatrix} \in\mathbb{R}^{2\times37}. \notag 
\end{align}

We construct another linear redundant constraint besides~\eqref{eq:qrr:model}. Write $y_{i,k}\in\mathbb R^{12}$ for the $k$-th block of $y_i$ and $A_{i,k}$ for the $k$-th row of $A_i$ such that the two rows of~\eqref{eq:qrr:model} read
$$
A_{i,k}z-e_3^\top P_iy_{i,k}=0,\quad k\in\{1,2\}.
$$
Since $y_i=v_i\otimes z$, we have $y_{i,k}=v_{i,k}z$. Multiplying the first row by $v_{i,2}$ and the second by $v_{i,1}$, then subtracting them gives the linear constraint
\begin{align*}
&v_{i,2}\big(A_{i,1}z-e_3^\top P_iy_{i,1}\big)-v_{i,1}\big(A_{i,2}z-e_3^\top P_iy_{i,2}\big)\\
&=A_{i,1}y_{i,2}-A_{i,2}y_{i,1}=0.
\end{align*}
Together with~\eqref{eq:qrr:model}, this gives three linear polynomials
\begin{align}
Q_ix^{(i)}&=0, \label{eq:qrr:Q}\\
Q_i&=\begin{bmatrix}0_{2\times1}&A_i&-I_2\otimes(e_3^\top P_i)\\0&0_{1\times12}&\begin{bmatrix}-A_{i,2}&A_{i,1}\end{bmatrix}\end{bmatrix}\in\mathbb R^{3\times37}. \notag
\end{align}

We next claim that $Q_i$ has rank 3, so that~\eqref{eq:qrr:Q} removes exactly three dimensions. We show that the 3 rows of the last 24 columns of $Q_i$,
$$
\begin{bmatrix}-e_3^\top P_i&0\end{bmatrix},\quad\begin{bmatrix}0&-e_3^\top P_i\end{bmatrix},\quad\begin{bmatrix}-A_{i,2}&A_{i,1}\end{bmatrix},
$$
are linearly independent. The first two are linearly independent since $e_3^\top P_i\neq0$. 
Suppose that the third is a linear combination of the first two, i.e., 
$$
\begin{bmatrix}-A_{i,2}&A_{i,1}\end{bmatrix}=a\begin{bmatrix}-e_3^\top P_i&0\end{bmatrix}+b\begin{bmatrix}0&-e_3^\top P_i\end{bmatrix}.
$$ 
Then we have $A_{i,1}=-b\,e_3^\top P_i$. Since $P_iz=Rp_i+t$, the last three entries of $e_3^\top P_i$ are $e_3^\top$, whereas those of $A_{i,1}$ are the first row of $\theta_ie_3^\top-\begin{bmatrix}I_2&0\end{bmatrix}K$, i.e., $(-f_x,-s,\theta_{i,1}-c_x)$. Comparing the first entries gives $f_x=0$, which contradicts $f_x\neq0$. Hence $Q_i$ has rank 3, as claimed.

Let $\Xi_i\in\mathbb R^{37\times34}$ be a matrix whose columns form a basis for $\ker(Q_i)$, chosen such that 
\begin{equation}\label{eq:qrr:xi:hub}
\Xi_{i,\mathbf h}=\begin{bmatrix}I_{13}&0\end{bmatrix}\in\mathbb R^{13\times(37-3)},
\end{equation}
where $\Xi_{i,\alpha}$ collects the rows of $\Xi_i$ indexed by the entries of the subvector $\alpha$ of $x^{(i)}$. Such a basis exists since the last 24 columns of $Q_i$ have rank 3. Since $x^{(i)}$ lies in the column space of $\Xi_i$, we can parameterize $x^{(i)}=\Xi_i\check x^{(i)}$ with a reduced variable $\check x^{(i)}\in\mathbb R^{34}$. Consequently, $X^{(i)}=x^{(i)}(x^{(i)})^\top$ can be parameterized by the smaller matrix $\check X^{(i)}=\check x^{(i)}(\check x^{(i)})^\top\in\mathbb S_+^{34}$ as
\begin{equation}\label{eq:qrr:param}
X^{(i)}=\Xi_i\check X^{(i)}\Xi_i^\top.
\end{equation}
Note that $\Xi_i$ is chosen such that the hub block is preserved, i.e., $X^{(i)}_{\mathbf{hh}}=\check X^{(i)}_{\mathbf{hh}}$, where $\check X^{(i)}_{\mathbf{hh}}$ is the top left $13\times13$ block of $\check X^{(i)}$. Substituting~\eqref{eq:qrr:param} into~\eqref{eq:chordal:sdp}, we have the following SDP.
\begin{equation}\label{eq:reduced:sdp}
\begin{aligned}
\min_{\check X^{(i)}\in\mathbb S_+^{34}}\ &\sum_{i=1}^N\mathrm{trace}(\Xi_{i,y_i}\check X^{(i)}\Xi_{i,y_i}^\top)\\
\text{s.t.}\quad &\tilde C_{g_{\theta,i},k}\bullet(\Xi_i\check X^{(i)}\Xi_i^\top)=0,\quad \begin{cases}
     k\in[74]\\i\in[N]
\end{cases}\\
&\tilde C_{r_j}\bullet\check X^{(1)}_{\mathbf{hh}}=0,\quad j\in[20]\\
&\mathrm{trace}(\check X^{(1)}_{zz})=1\\
&\Xi_{i,u_i}\check X^{(i)}\Xi_{i,u_i}^\top\in\mathcal{BS}^{36}_{12},\quad i\in[N]\\
&\check X^{(i)}_{\mathbf{hh}}=\check X^{(1)}_{\mathbf{hh}},\quad i=2,\dots,N
\end{aligned}\tag{Reduced-SDP}
\end{equation}

\begin{theorem}\label{thm:qrr}
Under Assumption~\ref{assump:nondegenerate}, given $\theta$ in a neighborhood of $\bar\theta$, \eqref{eq:reduced:sdp} is equivalent to \eqref{eq:chordal:sdp}.
\end{theorem}

The proof is given in Appendix~\ref{appx:proof:qrr}.
The quotient ring reduction removes 3 dimensions per clique, reducing each block from $\mathbb{S}^{37}$ to $\mathbb{S}^{34}$. 
We quantify the resulting speed-up in Section~\ref{sec:exp:scalability}.

\begin{remark}
Again, Assumption~\ref{assump:nondegenerate} is needed \textit{not} for the quotient ring reduction, but for the bijection in Definition~\ref{def:prob-equivalence}.
For example, if \eqref{eq:chordal:sdp} has an optimal solution with $Q_iX^{(i)}\neq0$ for some $i\in[N]$, then it is not of the form $\Xi_i\check X^{(i)}\Xi_i^\top$, and so it does not correspond to any optimal solution of \eqref{eq:reduced:sdp}.
\end{remark}

\subsection{PnP*: A certifiable solver for Perspective-n-Point}\label{sec:pnp-star}
We can now solve the relaxation \eqref{eq:reduced:sdp} computationally efficiently. Under Assumption~\ref{assump:nondegenerate} and given $\theta$ in a neighborhood of $\bar\theta$, if the solution of \eqref{eq:reduced:sdp} has rank-1 blocks, then we can recover the global minimizer of Problem~\ref{prob:mle} by Corollary~\ref{cor:stable}. However, verifying this requires checking the rank of all $N$ blocks of size $34$. The next Proposition shows that the $13\times13$ hub block alone suffices.

\begin{proposition}\label{prop:hub:rank}
Let $(\check X^{(1)},\dots,\check X^{(N)})$ be an optimal solution of \eqref{eq:reduced:sdp}. Then $\check X^{(i)}$ has rank 1 for all $i\in[N]$ if and only if $\check X^{(1)}_{\mathbf{hh}}$ has rank 1.
\end{proposition}

The proof is given in Appendix~\ref{appx:proof:hub:rank}. Nevertheless, recovering the global minimizer of Problem~\ref{prob:mle} from it still relies on Corollary~\ref{cor:stable}, and hence on $\theta$ being in a neighborhood of $\bar\theta$. 
This is because the feasible sets of Problem~\ref{prob:naive} and \ref{prob:lifted} contain points with $e_3^\top P_iz=0$ or $h=0$, which are excluded from being optimal only for $\theta$ near $\bar\theta$ (see Remark~\ref{remark:depth} and~\ref{remark:h}). In practice, however, we can check the two conditions on the solution directly, which removes the restriction on $\theta$.

\begin{proposition}\label{prop:recover}
Let $(\check X^{(1)},\dots,\check X^{(N)})$ be an optimal solution of \eqref{eq:reduced:sdp}. Suppose that $\check X^{(1)}_\mathbf{hh}$ has rank $1$, and write $\check X^{(1)}_\mathbf{hh}=(h,z)(h,z)^\top$. Then $(h,z)$ is the hub part of a global minimizer of Problem~\ref{prob:lifted}. Moreover, if $e_3^\top P_iz\neq0$ for all $i\in[N]$ and $h\neq0$, then $z/h$ is a global minimizer of Problem~\ref{prob:mle}.
\end{proposition}

The proof is given in Appendix~\ref{appx:proof:recover}.
Since Proposition~\ref{prop:recover} makes no assumption on $\theta$, we do not have to restrict $\theta$ to be in a neighborhood of $\bar\theta$. We now introduce our proposed solver, PnP$^*$, named after optimality and the star graph structure. We first solve~\eqref{eq:reduced:sdp} and take the eigenvector $(h,z)$ of $\check X^{(1)}_{\mathbf{hh}}$ with the largest eigenvalue. If $h=0$, then no pose can be extracted and we report failure. Otherwise, we extract $z^*=z/h=(\mathrm{vec}(R^*),t^*)$, and if $\check X^{(1)}_{\mathbf{hh}}$ has rank 1 and $e_3^\top P_iz^*\neq0$ for all $i\in[N]$, then $(R^*,t^*)$ is a certified global minimizer of Problem~\ref{prob:mle} by Proposition~\ref{prop:recover}. If either check fails, we project $R^*$ onto $\mathrm{SO}(3)$ and use it as an initialization for a local refinement by Levenberg--Marquardt. The complete procedure is given in Algorithm~\ref{alg:pnp}.

\begin{algorithm}[ht]
\caption{PnP* algorithm}\label{alg:pnp}
\begin{algorithmic}[1]
\Require $\{p_i\}_{i=1}^N$, $\{\theta_i\}_{i=1}^N$, and $K$
\Ensure $(R,t)\in\mathrm{SO}(3)\times\mathbb{R}^3$ or \texttt{NaN}
\State solve \eqref{eq:reduced:sdp} for $(\check X^{(1)},\dots,\check X^{(N)})$
\State $(h,z)\gets$ eigenvector of $\check X^{(1)}_{\mathbf{hh}}$ with the largest eigenvalue
\If{$h=0$}
    \State \Return \texttt{NaN}
\EndIf
\State extract $z^*\gets z/h$ and write $z^*=(\text{vec}(R^*),t^*)$
\If{$\check X^{(1)}_{\mathbf{hh}}$ is rank 1 \textbf{and} $e_3^\top P_iz^*\neq0$ for all $i$}
    \State \Return $(R^*,t^*)$ \Comment{certified pose}
\EndIf
\State $R\gets$ projection of $R^*$ onto $\mathrm{SO}(3)$
\State \Return local refinement initialized at $(R,t^*)$ 
\end{algorithmic}
\end{algorithm}

\section{Experimental results}\label{sec:experiment}
We use \citet{mosek} as our SDP solver and for the convex relaxations in BnB.
We determine the rank of a matrix by the number of its eigenvalues greater than $10^{-5}$.
All experiments are conducted over 40 Monte Carlo runs on an M4 MacBook Air with 16GB RAM. Our implementation is available online\footnote{\url{https://github.com/doggydoggy0101/pnp-star}}.

\subsection{Stability}\label{sec:exp:stability}
In this experiment, we validate the \textit{stability} result numerically, meaning our SDP relaxation is exact under a certain amount of noise. We set the number of points to $N=6$, the minimal configuration that admits a unique global solution under Assumption~\ref{assump:nondegenerate}.
We first generate the camera pose with a random rotation and a random translation with the norm less than or equal to 5.
For the camera intrinsic, we use the conventional setting~\citep{LMF09}
$$
K=\begin{pmatrix}
    800&0&320\\0&800&240\\0&0&1
\end{pmatrix}.
$$ 
We sample 2D measurements uniformly over the $640\times480$ image, then back-project them to 3D points with depths uniformly sampled in $[5,35]$. Finally, we add Gaussian noise to the 2D measurements with a standard deviation of $\sigma$ pixels.

We compare our PnP* against the minimal solver P3P~\citep{GHT03}, the local non-minimal solvers EPnP~\citep{LMF09} and SQPnP~\citep{TL20}, both with and without Levenberg--Marquardt (LM) refinement, the Alternative solver~\citep{SD20}, and branch-and-bound (BnB)~\citep{OKO06}. The Alternative solver is a certifiable solver for the weighted reprojection error, meaning \textit{it does not optimize the MLE objective}. In contrast, PnP* and BnB minimize the MLE directly. 

For P3P, we run all possible $3$-point combinations and keep the solution with the lowest reprojection error over all correspondences.
We run BnB with a relative suboptimality tolerance of $\varepsilon=10^{-4}$. Since BnB searches over the rotation space, it requires prior bounds on the translation such that the search space is compact. We compute an upper bound from the pose $(\breve R,\breve t)=(I_3,(0,0,1+\max_i\|p_i\|))$, whose depths $e_3^\top P_i\breve z$ are positive for all $i\in[N]$, and a lower bound by assuming that the reference point in BnB has depth at least 1, which holds since we sample depths in $[5,35]$.

\begin{figure*}[t]
    \centering
    \includegraphics[width=1\linewidth]{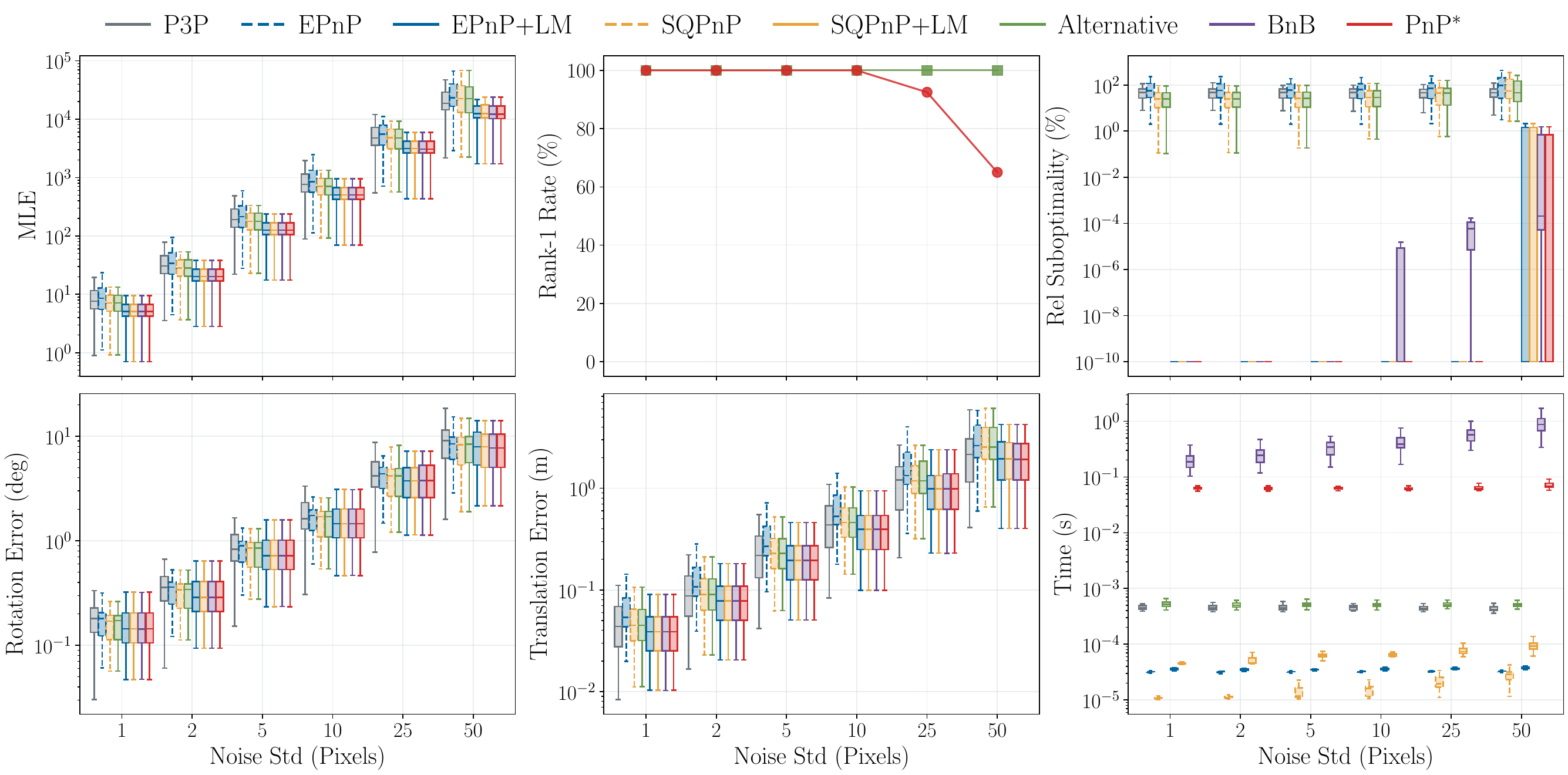}
    \caption{Noise sweep for $N=6$ with P3P~\citep{GHT03}, EPnP~\citep{LMF09}, SQPnP~\citep{TL20}, Alternative solver~\citep{SD20}, BnB~\citep{OKO06}, and our proposed PnP*.}
    \label{fig:exp1}
\end{figure*}

Figure~\ref{fig:exp1} reports the results of a noise sweep. The first plot of the top row shows the MLE objective of each method's solution. P3P never attains the optimum, even with the exhaustive selection over all 3-point combinations. SQPnP minimizes the same weighted reprojection error as the Alternative solver, and their results coincide. Interestingly, LM refinement empirically reaches the global minimizer regardless of its initialization. Nevertheless, PnP$^*$ provides a certificate that the local solvers actually work well. The remaining two plots in the top row show the rate of rank-1 solutions and the relative suboptimality bound $(\bar p-p^*_\mathrm{SDP})/p^*_\mathrm{SDP}$, where $\bar p$ is the MLE objective of a given solver's solution and $p^*_\mathrm{SDP}$ is the optimal value of~\eqref{eq:reduced:sdp}, a lower bound on the optimal value of Problem~\ref{prob:mle}. The relaxation is exact in all runs for $\sigma$ up to 10 pixels, and the bound certifies that LM refinement, BnB, and PnP$^*$ all attain the global minimizer under moderate noise. When the relaxation is not exact, PnP$^*$ falls back to LM refinement (Algorithm~\ref{alg:pnp}), and the bound still certifies that the refined solution is near-optimal in most runs.

The first two plots of the bottom row show the rotation error (in degrees) and the translation error (in meters). The MLE optimum also corresponds to a lower error in the estimated pose.
The last plot reports the solve time. PnP$^*$ is about three orders of magnitude slower than the local non-minimal solvers, while remaining $3\times$ to $13\times$ faster than the other global solver BnB.

\subsection{Scalability}\label{sec:exp:scalability}
We next compare the \textit{scalability} of solvers that optimize the MLE globally. Among prior work, BnB~\citep{OKO06} is the only tractable method that globally optimizes the MLE objective, so it is our sole external baseline.
Alongside it, we evaluate the three formulations \eqref{eq:sdp}, \eqref{eq:chordal:sdp}, and \eqref{eq:reduced:sdp}. 
We perform a sweep  over the number of points at a fixed noise of $\sigma=10$ pixels, which Figure~\ref{fig:exp1} shows to lie inside the neighborhood of the stable relaxation.
The remaining settings are the same as in Section~\ref{sec:exp:stability}.

\begin{figure*}[t]
    \centering
    \includegraphics[width=1\linewidth]{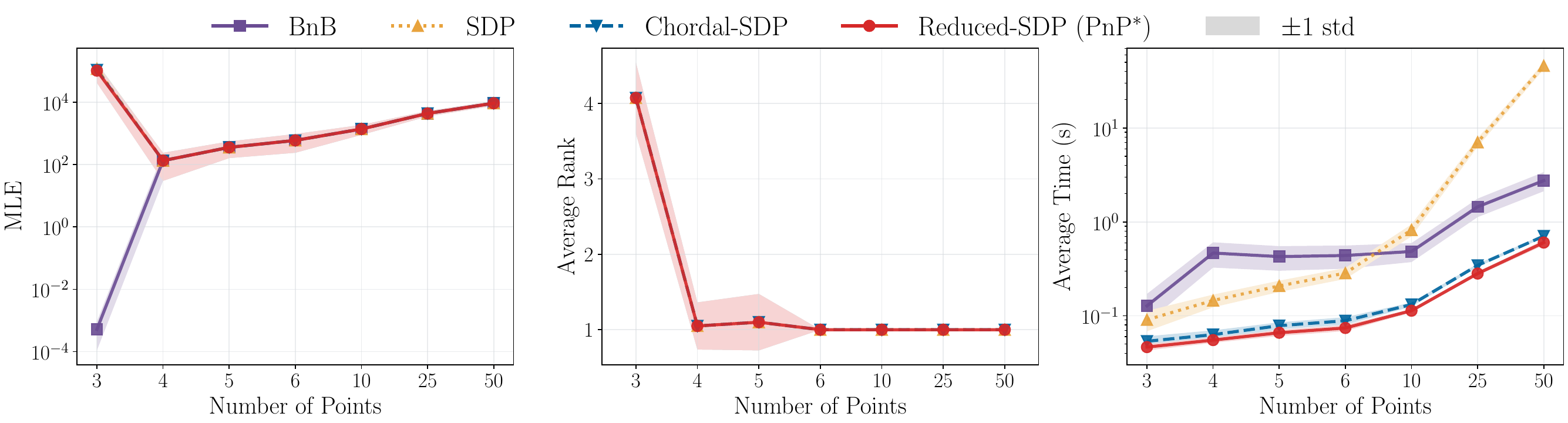}
    \caption{Point sweep for $\sigma=10$ pixels with BnB~\citep{OKO06}, \eqref{eq:sdp}, \eqref{eq:chordal:sdp}, and \eqref{eq:reduced:sdp}.}
    \label{fig:exp2}
\end{figure*}

Figure~\ref{fig:exp2} reports the results. The first two plots show the MLE objective and the rank of the SDP solution. Since the three formulations return the same solution in every run, their curves are the same. At $N=3$, the relaxation is tight, i.e., $p^*_\mathrm{SDP}$ agrees with the optimal value of BnB to numerical precision, but it is not exact since P3P admits up to four poses that fit the measurements exactly~\citep{FB81}. Problem~\ref{prob:mle} then has multiple optimal solutions, and the interior-point method returns a convex combination of the corresponding rank-1 solutions, and our solution extraction procedure fails to extract any of them. In this case, we observe a few instances violating the condition $h\neq0$ in Proposition~\ref{prop:recover}.
For $N\geq4$, all three formulations return rank 1 solutions on 39 of 40 runs at $N=4$, 37 of 40 runs at $N=5$, and all 40 runs for every $N\geq6$. This is consistent with Assumption~\ref{assump:nondegenerate}, which requires $N\geq6$ for the guarantee, while the relaxation is already exact in most runs for $N=4,5$.

The third plot reports the average runtime. Exploiting chordal sparsity is what makes the relaxation tractable, which \eqref{eq:chordal:sdp} is already $6\times$ faster than~\eqref{eq:sdp} at $N=10$, and the gap widens to about $65\times$ at $N=50$. The quotient ring reduction further makes~\eqref{eq:reduced:sdp} about $12$--$18\%$ faster than~\eqref{eq:chordal:sdp} at this range. Altogether, PnP$^*$ is between $3\times$ and $8\times$ faster than BnB across the sweep. At the minimal configuration $N=6$,~\eqref{eq:reduced:sdp} takes 74\,ms on average, a certified global solution is available at more than $10$ Hz.

\section{Conclusion}\label{sec:conclusion}
We present PnP*, the first scalable certifiable solver that optimizes the MLE of Perspective-n-Point directly. By introducing only necessary auxiliary variables and adding redundant constraints that preserve the problem's sparsity, we derive a convex SDP relaxation that is provably exact under moderate noise. By exploiting the problem's chordal sparsity structure, and together with quotient ring reduction, we reformulated an equivalent scalable SDP that returns a provably correct camera pose over 10 Hz at the minimal configuration. More broadly, this work extends scalable certifiable solvers to the perspective projection.

Future directions include incorporating outlier robustness for real-world application, and extending the relaxation technique to other large-scale vision problems such as camera calibration and bundle adjustment.

\begin{appendices}

\section{QCQP reformulations}\label{secA}
\setcounter{equation}{0}

\begin{definition}\label{def:well-defined}
Let $\mathcal{A}$ and $\mathcal{B}$ be sets, and consider the mapping $\phi:\mathcal A\to\mathcal B$. Then
\begin{enumerate}[(i)]
    \item $\phi$ is \textit{well-defined} if $\phi(a)\in\mathcal{B}$ for all $a\in\mathcal{A}$, and $a_1=a_2$ implies $\phi(a_1)=\phi(a_2)$ for all $a_1,a_2\in\mathcal{A}$.
    \item $\phi$ is \textit{injective} if $\phi(a_1)=\phi(a_2)$ implies $a_1=a_2$ for all $a_1,a_2\in\mathcal{A}$.
    \item $\phi$ is \textit{surjective} if for all $b\in\mathcal B$, there exists $a\in\mathcal A$ such that $\phi(a)=b$.
    \item $\phi$ is \textit{bijective} if it is both injective and surjective.
\end{enumerate}
\end{definition}

\begin{lemma}\label{lem:homogenization}
\eqref{eq:qcqp:inhomo} and \eqref{eq:qcqp:homo} are equivalent.
\end{lemma}

\begin{proof}
Let $f_j(\tilde x)=\tilde x^\top P_j\tilde x+2q_j^\top\tilde x+r_j$ for $j=0,\dots,\ell$, and denote the feasible sets of \eqref{eq:qcqp:inhomo} and \eqref{eq:qcqp:homo}, respectively, by
\begin{align*}
  \mathcal{F} &\triangleq \{\tilde x\mid f_j(\tilde x)=0,\ j\in[\ell]\},\\
  \mathcal{F}_\text{H} &\triangleq\{x=(h,\tilde x)\mid x^\top C_jx=0,\ j\in[\ell],\ h^2=1\}.
\end{align*}
By Definition~\ref{def:sign-equivalence}, denote by $[\mathcal{F}_\text{H}]\triangleq\{[x]\mid x\in\mathcal{F}_\text{H}\}$ its set of equivalence classes. Consider the mapping
\begin{equation*}
  \phi:[\mathcal{F}_\text{H}]\to\mathcal{F},\quad [(h,\tilde x)]\mapsto\tilde x/h.
\end{equation*}

We first show that $\phi$ is well-defined. Let $x=(h,\tilde x)\in\mathcal{F}_\text{H}$, then $\phi([x])=\tilde x/h$ is defined since $h^2=1$ implies $h\neq0$. Next, by \eqref{eq:inhomo2homo} and $h^2=1$, we have
\begin{equation}\label{eq:homo-identity}
\begin{aligned}
  x^\top C_jx&=\tilde x^\top P_j\tilde x+2hq_j^\top\tilde x+r_jh^2\\
  &= h^2 f_j(\tilde x/h) = f_j(\tilde x/h)
  \end{aligned}
\end{equation}
for $j=0,\dots,\ell$. Therefore, $x^\top C_jx=f_j(\tilde x/h)=0$ for all $j\in[\ell]$, and so $\phi([x])\in\mathcal F$.
Finally, let $[x_1]=[x_2]$. Since $x_2=\pm x_1$, we have $\phi([x_2])=(\pm\tilde x_1)/(\pm h_1)=\tilde x_1/h_1=\phi([x_1])$.

We next show that $\phi$ is a bijection (between feasible sets). For injectivity, let $\phi([x_1])=\phi([x_2])$, i.e., $\tilde x_1/h_1=\tilde x_2/h_2$. Since $h_1,h_2\in\{\pm1\}$, $x_2=\pm x_1$, i.e., $[x_1]=[x_2]$. For surjectivity, let $\tilde x\in\mathcal{F}$ and $x=(1,\tilde x)$. Since $x^\top C_jx=f_j(\tilde x)=0$ for all $j\in[\ell]$ by \eqref{eq:homo-identity} and $h^2=1$, $[x]\in[\mathcal{F}_\text{H}]$ and $\phi([x])=\tilde x$.

Furthermore, $\phi$ preserves the objective value since $x^\top C_0x=f_0(\tilde x/h)=f_0(\phi([x]))$ by \eqref{eq:homo-identity}. Therefore, $\phi$ is a bijection between the equivalence classes of optimal sets, and the two problems are equivalent by Definition~\ref{def:prob-equivalence}.
\end{proof}

\subsection{Proof of Proposition~\ref{prop:mle:unique}}\label{appx:proof:mle:unique}

\begin{proof}
Consider Problem~\ref{prob:mle} at $\bar\theta$. Note that by \eqref{eq:noiseless}, $z_\mathrm{gt}$ attains zero objective value, which is the lowest possible for any sum of squares objective, hence $z_\mathrm{gt}$ is a global minimizer.

Let $z^*$ be a global minimizer. Since $z_\mathrm{gt}$ has zero objective value, $z^*$ also has to attain zero objective value, hence we have
$$
\bar\theta_i-\pi_i(z^*)=\bar\theta_i-\frac{\begin{bmatrix}I_2&0\end{bmatrix}KP_iz^*}{e_3^\top P_iz^*}=0
$$
for all $i\in[N]$. Multiplying by $e_3^\top P_iz^*$ gives $(\bar\theta_ie_3^\top-\begin{bmatrix}I_2&0\end{bmatrix}K)P_iz^*=0$. By stacking the equations for $N$ points, we have the linear system $A(\bar\theta)z^*=0$, where $A:\mathbb{R}^{2N}\to\mathbb{R}^{2N\times 12}$ is an affine mapping defined by
\begin{equation}\label{eq:affine}
A(\theta)\triangleq
\begin{pmatrix}
    \vdots  \\
    (\theta_ie_3^\top-\begin{bmatrix}I_2&0\end{bmatrix}K)P_i \\
    \vdots
\end{pmatrix}\in\mathbb{R}^{2N\times12}.
\end{equation}

That is, every global minimizer has to lie in $\ker(A(\bar\theta))$. To determine the kernel of $A(\bar\theta)$, we first consider its rank. Since $z_\mathrm{gt}\in\ker(A(\bar\theta))$ and $z_\text{gt}\neq0$, by rank-nullity theorem, $\text{rank}(A(\bar\theta))\leq 11$.
If $\text{rank}(A(\bar\theta))<11$, then there exists more than 1 (vectorized) camera resectioning solution that satisfy the linear system $A(\bar\theta)z^*=0$, which contradicts Assumption~\ref{assump:nondegenerate}~\citep[Result 22.6]{HZ03}\footnote{If $N\leq 5$, then $\text{rank}(A(\bar\theta))\leq 2N\leq10$, hence $N\geq 6$ is necessary in Assumption~\ref{assump:nondegenerate}.}. 
Altogether, $\text{rank}(A(\bar\theta))=11$ and $\dim(\ker(A(\bar\theta)))=1$.

Since $z_\mathrm{gt}\in\ker(A(\bar\theta))$, we can write $\ker(A(\bar\theta))=\text{span}\{z_\text{gt}\}$.
Let $z^*=cz_\text{gt}$ for some scalar $c$, and denote $R^*=cR_\text{gt}$ the rotation component. We next determine $c$ by the rotation constraint $R^*\in\text{SO}(3)$. Since rotations are orthonormal, $(R^*)^\top R^*=c^2I_3=I_3$ and hence $c=\pm1$. Since rotations have determinant 1, $\det(R^*)=c^3=1$ forces $c=1$. Therefore, $z^*=z_\text{gt}$ is the unique minimizer.
\end{proof}

\subsection{Rotation constraints}\label{appx:rot}
Denote by $c^{(i)}\in\mathbb{R}^3$ the $i$-th column of $R$ such that $z=(c^{(1)},c^{(2)},c^{(3)},t)$, and by
$c^{(i)}_k$ its $k$-th entry.
Column orthonormality gives (6 constraints)
\begin{align}
  \|c^{(1)}\|^2&=h^2, \label{eq:rot:col1}\\
  \|c^{(2)}\|^2&=h^2, \label{eq:rot:col2}\\
  \|c^{(3)}\|^2&=h^2, \label{eq:rot:col3}\\
  (c^{(1)})^\top c^{(2)}&=0, \notag\\
  (c^{(1)})^\top c^{(3)}&=0, \notag\\
  (c^{(2)})^\top c^{(3)}&=0. \notag
\end{align}
Row orthonormality gives (6 constraints)
\begin{align}
  \textstyle\sum_{k=1}^3(c^{(k)}_1)^2&=h^2, \label{eq:rot:row1}\\
  \textstyle\sum_{k=1}^3(c^{(k)}_2)^2&=h^2, \label{eq:rot:row2}\\
  \textstyle\sum_{k=1}^3(c^{(k)}_3)^2&=h^2, \label{eq:rot:row3}\\
  \textstyle\sum_{k=1}^3c^{(k)}_1c^{(k)}_2&=0, \notag\\
  \textstyle\sum_{k=1}^3c^{(k)}_1c^{(k)}_3&=0, \notag\\
  \textstyle\sum_{k=1}^3c^{(k)}_2c^{(k)}_3&=0. \notag
\end{align}
Determinant gives (9 constraints)
\begin{align}
  c^{(2)}\times c^{(3)}&=h\,c^{(1)}, \label{eq:rot:det1}\\
  c^{(3)}\times c^{(1)}&=h\,c^{(2)}, \label{eq:rot:det2}\\
  c^{(1)}\times c^{(2)}&=h\,c^{(3)}, \label{eq:rot:det3}
\end{align}
each of which consists of 3 scalar constraints.

Note that summing \eqref{eq:rot:col1}--\eqref{eq:rot:col3} and summing
\eqref{eq:rot:row1}--\eqref{eq:rot:row3} both give
$\sum_{k=1}^3\|c^{(k)}\|^2=3h^2$, hence we can write \eqref{eq:rot:row3} as a linear combination of the other constraints, and so we omit it. The remaining 20 constraints are linearly independent, and we denote them by $r_j(h,z)=0$ for $j\in[20]$.

\subsection{Proof of Theorem~\ref{thm:naive}}\label{appx:proof:naive}
Denote the feasible sets of Problem~\ref{prob:mle} by 
\begin{equation}\label{eq:feasible-set:mle}
\begin{aligned}
    \mathcal Z\triangleq\{z\in\mathbb{R}^{12}&\mid e_3^\top P_iz\neq0,\ i\in[N],\\ 
    &\quad R\in\mathrm{SO}(3)\},
\end{aligned}
\end{equation}
where $z=(\text{vec}(R),t)$, and the feasible set of Problem~\ref{prob:naive} by
\begin{equation}\label{eq:feasible-set:naive}
\begin{aligned}
    \mathcal W_\theta\triangleq\{w\in\mathbb{R}^{13+2N}&\mid m_{\theta,i}(z,v)=0,\ i\in[N], \\
    &\quad r_j(h,z)=0,\ j\in[20], \\
    &\quad h^2=1\},
\end{aligned}
\end{equation}
where $w=(h,z,v)$. Let
\begin{equation}\label{eq:feasible-set:naive:restricted}
  \tilde{\mathcal W}_\theta\triangleq\{w\in\mathcal W_\theta\mid
  e_3^\top P_iz\neq0,\ i\in[N]\}.
\end{equation}
Note that $\tilde{\mathcal W}_\theta\subsetneq\mathcal W_\theta$ since Problem~\ref{prob:naive} admits feasible points with $e_3^\top P_iz=0$. 
By construction of $v$ in \eqref{eq:aux:v}, $(h,z,v)\in\mathcal W_\theta\iff(h,z,v-\theta+\bar\theta)\in\mathcal W_{\bar\theta}$. That is, 
\begin{equation}\label{eq:set-translation}
\mathcal W_\theta=\mathcal W_{\bar\theta}+(0,0,\theta-\bar\theta).
\end{equation}

The goal is to construct a bijection between the optimal sets of Problem~\ref{prob:mle} and Problem~\ref{prob:naive}.
We first show that for $\theta$ in a neighborhood of $\bar\theta$, the optimal set of Problem~\ref{prob:naive} is in $\tilde{\mathcal W}_\theta$ such that it suffices to consider $\tilde{\mathcal W}_\theta$.

\begin{lemma}\label{lemma:feasible-set:naive:restricted}
Under Assumption~\ref{assump:nondegenerate}, given $\theta$ in a neighborhood of $\bar\theta$, every optimal point of Problem~\ref{prob:naive} lies in $\tilde{\mathcal W}_\theta$.
\end{lemma}

\begin{proof}
Let $(h,z,v)\in\mathcal W_{\bar\theta}\setminus\tilde{\mathcal W}_{\bar\theta}$, i.e., $e_3^\top P_kz=0$ for some $k$. We claim that $v\neq0$. If $v=0$, then the model constraints~\eqref{eq:model} read $A(\bar\theta)z=0$, and so $z=cz_\mathrm{gt}$ by the proof of Proposition~\ref{prop:mle:unique}. Since $r_j(1,z/h)=r_j(h,z)/h^2=0$, we have $(c/h)R_\mathrm{gt}\in\mathrm{SO}(3)$, which forces $c=h$. Hence $e_3^\top P_kz=h\,e_3^\top P_kz_\mathrm{gt}\neq0$, a contradiction. 

Next, let $\delta>0$ be the lower bound of $\|v\|$ over all $(h,z,v)\in\mathcal W_{\bar\theta}\setminus\tilde{\mathcal W}_{\bar\theta}$, and $\|\theta-\bar\theta\|<\delta/2$. Since $(1,z_\mathrm{gt},0)\in\mathcal W_{\bar\theta}$, we have $(1,z_\mathrm{gt},\theta-\bar\theta)\in\mathcal W_\theta$ by~\eqref{eq:set-translation}, hence $\|\theta-\bar\theta\|^2$ is an upper bound of the optimal value. Let $(\hat h,\hat z,\hat v)\in\mathcal W_\theta\setminus\tilde{\mathcal W}_\theta$. By~\eqref{eq:set-translation}, $(\hat h,\hat z,\hat v-\theta+\bar\theta)\in\mathcal W_{\bar\theta}\setminus\tilde{\mathcal W}_{\bar\theta}$, then by triangle inequality,
$$
\|\hat v\|\geq\|\hat v-\theta+\bar\theta\|-\|\theta-\bar\theta\|\geq\delta-\|\theta-\bar\theta\|>\|\theta-\bar\theta\|. 
$$
Hence $(\hat h,\hat z,\hat v)$ is not optimal at $\theta$, and so every optimal point of Problem~\ref{prob:naive} at $\theta$ lies in $\tilde{\mathcal W}_{\theta}$.
\end{proof}

We proceed to prove Theorem~\ref{thm:naive}.

\begin{proof}
The equivalence class\footnote{The model constraints~\eqref{eq:model} are not homogenized for notation simplicity, hence only $h$ and $z$ are sign-equivalent.} of a feasible point of Problem~\ref{prob:naive} is $[(h,z,v)]\triangleq\{(h,z,v),(-h,-z,v)\}$. We denote by $[\tilde{\mathcal W}_\theta]$ the set of equivalence classes of $\tilde{\mathcal W}_\theta$. 
Consider the mapping
\begin{equation}\label{eq:thm1:phi}
  \phi:\mathcal Z\to[\tilde{\mathcal W}_\theta],\quad
  z\mapsto[(1,z,v)],
\end{equation}
where $v_i=\theta_i-\pi_i(z)$ for $i\in[N]$.

To show $\phi$ is well-defined, let $z\in\mathcal Z$ and denote $w\triangleq(1,z,v)$. Since $e_3^\top P_iz\neq0$, $v$ is also defined. 
Multiplying $v_i=\theta_i-\pi_i(z)$ by $e_3^\top P_iz$ gives $m_{\theta,i}(z,v)=0$, and so the model constraints are satisfied. By Appendix~\ref{appx:rot}, $R\in\mathrm{SO}(3)$ gives $r_j(1,z)=0$, and so the rotation constraints are also satisfied. The homogeneous constraint holds trivially, and since $e_3^\top P_iz\neq0$, we have $w\in\tilde{\mathcal W}_\theta$ and $\phi(z)\in[\tilde{\mathcal W}_\theta]$. Finally, $z_1=z_2$ trivially implies $\phi(z_1)=\phi(z_2)$.

We next show that $\phi$ is a bijection. For injectivity, let $\phi(z_1)=\phi(z_2)$. 
Since $(1,z_1,v_1)=(-1,-z_2,v_2)$ is impossible, we have $(1,z_1,v_1)=(1,z_2,v_2)$, and hence $z_1=z_2$.
For surjectivity, let $[w]\in[\tilde{\mathcal W}_\theta]$ with $w=(h,z,v)$, and let $z'\triangleq z/h$. Denote the rotation components of $z$ and $z'$ by $R$ and $R'$, respectively. Since $r_j$ is homogeneous quadratic and $h\neq0$, $r_j(1,z')=r_j(h,z)/h^2=0$ for all $j\in[20]$, and so $R'\in\mathrm{SO}(3)$ by Appendix~\ref{appx:rot}. Since $e_3^\top P_iz'=e_3^\top P_iz/h\neq0$, we have $z'\in\mathcal Z$. Moreover, $m_{\theta,i}(z',v)=m_{\theta,i}(z,v)/h=0$, dividing by $e_3^\top P_iz'\neq0$ gives $v_i=\theta_i-\pi_i(z')$. Hence $\phi(z')=[(1,hz,v)]=[w]$, since $z'=z/h=hz$ and $(1,hz,v)=(-h,-z,v)$ when $h=-1$.

Finally, $\phi$ preserves the objective value since the objective of Problem~\ref{prob:naive} at $\phi(z)$ is $\|v\|^2=\sum_{i=1}^N\|\theta_i-\pi_i(z)\|^2$, which is the objective of Problem~\ref{prob:mle} at $z$. Since every optimal point of Problem~\ref{prob:naive} lies in $\tilde{\mathcal W}_{\theta}$ by Lemma~\ref{lemma:feasible-set:naive:restricted}, $\phi$ is a bijection between the equivalence classes of the optimal sets, and Problem~\ref{prob:mle} and Problem~\ref{prob:naive} are equivalent by Definition~\ref{def:prob-equivalence}.
\end{proof}

\subsection{Proof of Corollary~\ref{cor:naive}}\label{appx:proof:naive:unique}
\begin{proof}
By Proposition~\ref{prop:mle:unique}, Problem~\ref{prob:mle} has a unique global minimizer $z_\text{gt}$. By the proof of Theorem~\ref{thm:naive}, Problem~\ref{prob:naive} has the unique equivalence class of global minimizer $\phi(z_\mathrm{gt})=[(1,z_\mathrm{gt},0)]$, where $\phi$ is defined in \eqref{eq:thm1:phi}, and hence $\bar w=(1,z_\mathrm{gt},0)$.
\end{proof}

\subsection{Non-informative relaxation of Problem~\ref{prob:naive} }
\begin{proposition}\label{prop:non-informative}
For any $\theta$, the \eqref{eq:sdp:shor} relaxation of Problem~\ref{prob:naive} has an optimal value of zero.
\end{proposition}

\begin{proof}
We first homogenize the model constraints~\eqref{eq:model} as 
\begin{align*}
m^h_{\theta,i}(h,z,v)&\triangleq h\big(\theta_i(e_3^\top P_iz)-\begin{bmatrix}I_2&0\end{bmatrix}KP_iz\big)\\
&\quad -v_i(e_3^\top P_iz)=0.
\end{align*}
Following Section~\ref{sec:shor}, the \eqref{eq:sdp:shor} relaxation of Problem~\ref{prob:naive} is
\begin{equation}\label{eq:sdp:naive}
\begin{aligned}
\min_{W\in\mathbb S^{13+2N}_+} &\mathrm{tr}(W_{vv})\\
\text{s.t.}\quad\ &C_{m^h_{\theta,i}}\bullet W=0,\quad i\in[N]\\
&C_{r_j}\bullet W=0,\quad j\in[20]\\
&W_{hh}=1
\end{aligned}
\end{equation}
where the blocks of $W$ are indexed by $(h,z,v)$. Let
$$
W^*\triangleq\begin{pmatrix}
1&0&0\\
0&\frac13I_{12}&0\\
0&0&0
\end{pmatrix}\in\mathbb{S}^{13+2N}_+.
$$

We claim that $W^*$ is optimal for~\eqref{eq:sdp:naive}. The model constraints involve only the blocks $W_{hz}$ and $W_{vz}$, which are zero in $W^*$, hence satisfied. For the rotation constraints, \eqref{eq:rot:col1}--\eqref{eq:rot:col3} and~\eqref{eq:rot:row1}--\eqref{eq:rot:row2} hold since the 3 corresponding diagonal entries of $W^*_{zz}$ sum to $1=W^*_{hh}$. The remaining rotation constraints involve only $W_{hz}$ and off-diagonal entries of $W_{zz}$, which are zero, hence also satisfied. Finally, the objective value at $W^*$ is $\mathrm{tr}(W^*_{vv})=0$, and since $\mathrm{tr}(W_{vv})\geq0$ for any $W\succeq0$, $W^*$ is optimal regardless of $\theta$.
\end{proof}

\subsection{Proof of Theorem~\ref{thm:lifted}}\label{appx:proof:lifted}
Denote the feasible set of Problem~\ref{prob:lifted} by
\begin{equation}\label{eq:feasible-set:lifted}
\begin{aligned}
    \mathcal X_\theta\triangleq\{x\in\mathbb{R}^{13+24N}&\mid g_{\theta,i}(x)(x^{(i)})^\top=0,\ i\in[N],\\
    &\quad r_j(h,z)=0,\ j\in[20],\\
    &\quad \|z\|^2=1,\\
    &\quad u_iu_i^\top\in\mathcal{BS}^{36}_{12},\ i\in[N]\},
\end{aligned}
\end{equation}
where $x=(h,z,y)$ and $u_i=(z,y_i)$, and let 
\begin{equation}\label{eq:feasible-set:lifted:restricted}
\tilde{\mathcal X}_\theta\triangleq\{x\in\mathcal X_\theta\mid h\neq0\}.
\end{equation}
Note that $\tilde{\mathcal X}_\theta\subsetneq\mathcal X_\theta$ since Problem~\ref{prob:lifted} admits feasible points with $h=0$. Again, we first show that for $\theta$ in a neighborhood of $\bar\theta$, the optimal set of Problem~\ref{prob:lifted} is in $\tilde{\mathcal X}_\theta$.

\begin{lemma}\label{lemma:feasible-set:lifted:restricted}
Under Assumption~\ref{assump:nondegenerate}, given $\theta$ in a neighborhood of $\bar\theta$, every optimal point of Problem~\ref{prob:lifted} lies in $\tilde{\mathcal X}_\theta$.
\end{lemma}

\begin{proof}
Let $x=(0,z,y)\in\mathcal X_\theta\setminus\tilde{\mathcal X}_\theta$. Since $z\neq0$, the block symmetry constraints give $y_i=v_i\otimes z$ for some $v_i\in\mathbb R^2$ \citep[Section 3.3]{C21}. Since $x^{(i)}$ contains $z\neq0$ and $(I_2\otimes(e_3^\top P_i))(v_i\otimes z)=v_i(e_3^\top P_iz)$, the lifted model constraints give $g_{\theta,i}(x)=m_{\theta,i}(z,v)=0$. Since $h=0$, summing \eqref{eq:rot:col1}--\eqref{eq:rot:col3} gives $z=(0,t)$ with $\|t\|=1$ and so $P_iz=t$. Therefore, the model constraints \eqref{eq:model} read $(\theta_i-v_i)t_3=[I_2\ 0]Kt$. If $t_3=0$, then $Kt=0$ since $e_3^\top Kt=t_3$, i.e., $t=0$, a contradiction. Hence $v_i=\theta_i-q$ for all $i\in[N]$, where $q\triangleq[I_2\ 0]Kt/t_3$, i.e., all world points are projected to the same pixel point.

Next, let $\delta\triangleq\min_{q\in\mathbb R^2}\|\bar\theta-(q,\dots,q)\|$. We claim that $\delta>0$. If $\delta=0$, then $\bar\theta_i=q$ for all $i\in[N]$ and some $q\in\mathbb R^2$. Let $t\triangleq K^{-1}(q,1)/\|K^{-1}(q,1)\|$, then $[I_2\ 0]Kt=q\,t_3$, and so $z=(0,t)$ gives $A(\bar\theta)z=0$ by~\eqref{eq:affine}. By the proof of Proposition~\ref{prop:mle:unique}, $z=cz_\mathrm{gt}$ for some scalar $c$. Since the rotation component of $z$ is zero, we have $cR_\mathrm{gt}=0$, which forces $c=0$ and contradicts $t\neq0$. Let $\|\theta-\bar\theta\|<\delta/2$. Since $(1,z_\mathrm{gt},\theta-\bar\theta)\in\mathcal W_\theta$ by \eqref{eq:set-translation}, we have $(1,z_\mathrm{gt},(\theta-\bar\theta)\otimes z_\mathrm{gt})/\|z_\mathrm{gt}\|\in\mathcal X_\theta$, hence $\|\theta-\bar\theta\|^2$ is an upper bound of the optimal value. Since $\|z\|=1$, the objective at $x$ is $\|v\|^2$, and by triangle inequality,
\begin{align*}
\|v\|=\|\theta-(q,\dots,q)\|&\geq\|\bar\theta-(q,\dots,q)\|-\|\theta-\bar\theta\|\\
&\geq\delta-\|\theta-\bar\theta\|>\|\theta-\bar\theta\|.
\end{align*}
Hence $x$ is not optimal at $\theta$, and so every optimal point of Problem~\ref{prob:lifted} at $\theta$ lies in $\tilde{\mathcal X}_\theta$.
\end{proof}

We proceed to prove Theorem~\ref{thm:lifted}.

\begin{proof}
Recall from Appendix~\ref{appx:proof:naive} the feasible set~\eqref{eq:feasible-set:naive} of Problem~\ref{prob:naive} and its equivalence class $[(h,z,v)]=\{(h,z,v),(-h,-z,v)\}$, and denote by $[\mathcal W_\theta]$ the set of equivalence classes of $\mathcal W_\theta$. 
Since Problem~\ref{prob:lifted} is a homogeneous QCQP, its equivalence class is $[x]=\{\pm x\}$ as in Definition~\ref{def:sign-equivalence}, we denote by $[\tilde{\mathcal X}_\theta]$ the set of equivalence classes of $\tilde{\mathcal X}_\theta$. Consider the mapping
\begin{equation}\label{eq:thm2:phi}
\phi:[\mathcal W_\theta]\to[\tilde{\mathcal X}_\theta],\quad[(h,z,v)]\mapsto[(h,z,v\otimes z)/\|z\|].
\end{equation}

To show $\phi$ is well-defined, let $w=(h,z,v)\in\mathcal W_\theta$ and denote $x\triangleq(h,z,v\otimes z)/\|z\|$, which is defined since the rotation constraints with $h^2=1$ give $z\neq0$. By the proof of Lemma~\ref{lemma:feasible-set:lifted:restricted}, $g_{\theta,i}(x)=m_{\theta,i}(z,v)/\|z\|=0$, which implies $g_{\theta,i}(x)(x^{(i)})^\top=0$, and so the lifted model constraints are satisfied. Since $r_j$ is homogeneous quadratic, $r_j(h/\|z\|,z/\|z\|)=r_j(h,z)/\|z\|^2=0$, and the normalization constraint holds trivially. The block symmetry constraints hold since $y_i/\|z\|=v_i\otimes z/\|z\|$~\citep[Section 3.3]{C21}. Since $h/\|z\|\neq0$, we have $x\in\tilde{\mathcal X}_\theta$. Finally, $(-h,-z,v)$ maps to $-x\in[x]$, and so $\phi([w])\in[\tilde{\mathcal X}_\theta]$.

We next show that $\phi$ is a bijection. For injectivity, let $\phi([w_1])=\phi([w_2])$, i.e., $(h_1,z_1,v_1\otimes z_1)/\|z_1\|=\pm(h_2,z_2,v_2\otimes z_2)/\|z_2\|$. Since $h_1^2=h_2^2=1$, the first component gives $\|z_1\|=\|z_2\|$, and so $h_1=\pm h_2$ and $z_1=\pm z_2$. The third component then gives $v_1\otimes z_1=\pm(v_2\otimes z_2)=v_2\otimes(\pm z_2)=v_2\otimes z_1$, and since $z_1\neq0$, we have $v_1=v_2$. Altogether, $w_1\in[w_2]$. For surjectivity, let $[x]\in[\tilde{\mathcal X}_\theta]$ with $x=(h,z,y)$. By the proof of Lemma~\ref{lemma:feasible-set:lifted:restricted}, $y_i=v_i\otimes z$ and $m_{\theta,i}(z,v)=0$. Let $w\triangleq(h/|h|,z/|h|,v)$, which is defined since $h\neq0$. Then $(h/|h|)^2=1$, $r_j(h/|h|,z/|h|)=r_j(h,z)/h^2=0$, and $m_{\theta,i}(z/|h|,v)=m_{\theta,i}(z,v)/|h|=0$, hence $w\in\mathcal W_\theta$. Since $\|z/|h|\|=1/|h|$, we have $\phi([w])=[(h,z,v\otimes z)]=[x]$.

Finally, $\phi$ preserves the objective value since the objective of Problem~\ref{prob:lifted} at $\phi([w])$ is $\|v\otimes z\|^2/\|z\|^2=\|v\|^2$, which is the objective of Problem~\ref{prob:naive} at $w$. Since every optimal point of Problem~\ref{prob:lifted} lies in $\tilde{\mathcal X}_\theta$ by Lemma~\ref{lemma:feasible-set:lifted:restricted}, $\phi$ is a bijection between the equivalence classes of the optimal sets, and Problem~\ref{prob:naive} and Problem~\ref{prob:lifted} are equivalent by Definition~\ref{def:prob-equivalence}.
\end{proof}

\subsection{Proof of Corollary~\ref{cor:lifted}}\label{appx:proof:lifted:unique}
\begin{proof}
By Corollary~\ref{cor:naive}, Problem~\ref{prob:naive} has the unique equivalence class of global minimizer $[\bar w]$ with $\bar w=(1,z_\mathrm{gt},0)$. 
By the proof of Theorem~\ref{thm:lifted}, Problem~\ref{prob:lifted} has the unique equivalence class of global minimizer $\phi([\bar w])=[(1/\|z_\mathrm{gt}\|,z_\mathrm{gt}/\|z_\mathrm{gt}\|,0)]$, where $\phi$ is defined in \eqref{eq:thm2:phi}. Let $\bar x=(\bar h,\bar z,0)=(1/\|z_\mathrm{gt}\|,z_\mathrm{gt}/\|z_\mathrm{gt}\|,0)$. Then $\bar z/\bar h=z_\mathrm{gt}$ is the unique global minimizer of Problem~\ref{prob:mle}.
\end{proof}

\section{Stability of \ref{eq:sdp}}\label{secB}
\setcounter{equation}{0}
Consider a QCQP
\begin{align*}
\min_x\ &f_0(x)\\
\mathrm{s.t.}\ &f_i(x)=0,\quad i\in[\ell]
\end{align*}
where $f_i$ for $i=0,\dots,\ell$ are quadratic functions.
Denote the \textit{dual slack variable} by
\begin{equation}\label{eq:dual-slack-variable}
S(\lambda) \triangleq \nabla^2 f_0(x)+\sum_{i=1}^{\ell} \lambda_i\nabla^2 f_i(x),
\end{equation}
where $\lambda\in\mathbb{R}^\ell$ is the Lagrange multiplier.
We say $(x,\lambda)$ is primal-dual optimal with zero duality gap~\citep[Lemma 2.1]{CAP22} if
\begin{enumerate}[(i)]
    \item (primal feasible) $f_i(x)=0$ for $i\in[\ell]$,
    \item (dual feasible) $S(\lambda)\succeq0$,
    \item (stationary) $S(\lambda)x=0$.
\end{enumerate}

Denote by $\mathcal X_\theta$ (see \eqref{eq:feasible-set:lifted}) the feasible set of Problem~\ref{prob:lifted}. We assume that Assumption~\ref{assump:nondegenerate} holds such that Problem~\ref{prob:lifted} at $\bar\theta$ has a unique equivalence class of global minimizer $[\bar x]$ by Corollary~\ref{cor:lifted}, where $\bar x\in\mathcal X_{\bar\theta}$. 
Note that $(\bar x,0)$ is primal-dual optimal with zero duality gap for Problem~\ref{prob:lifted} at $\bar\theta$. Denote by
\begin{equation}\label{eq:slack}
    \bar S\triangleq S(0)=\begin{pmatrix}0&0&0\\0&0&0\\0&0&2I_{24N}\end{pmatrix}\in\mathbb S^{13+24N}
\end{equation}
the dual slack variable in \eqref{eq:dual-slack-variable}. Then $\bar x\in\mathcal X_{\bar\theta}$ (primal feasible), $\bar S\succeq0$ (dual feasible), and $\bar S\bar x=0$ (stationary).
We build on the stability conditions by \citet{CAP22} as follows. 

\begin{lemma}[{\citet[Theorem 4.5]{CAP22}}]\label{lem:Cifuentes}
\eqref{eq:sdp} is stable near $\bar\theta$ if
\begin{enumerate}[(i)]
    \item \textnormal{(Appendix~\ref{appx:ACQ})} the Abadie constraint qualification holds at $\bar x\in\mathcal X_{\bar\theta}$,
    \item \textnormal{(Appendix~\ref{appx:smoothness})} the mapping $\theta\mapsto\mathcal X_\theta$ is smooth nearby $(\bar\theta,\bar x)$,
    \item \textnormal{(Appendix~\ref{appx:non-branch-point})} $\bar x$ is not a branch point of $\mathcal X_{\bar\theta}$ with respect to the mapping $x\mapsto \bar Sx$,
    \item \textnormal{(Appendix~\ref{appx:restricted-Slater})} the restricted Slater condition holds at $(\bar x,0)$.
\end{enumerate}
\end{lemma}

The rest of this section shows the four conditions.
The Abadie constraint qualification (ACQ) guarantees regularity of $\mathcal{X}_{\bar\theta}$. Roughly speaking, the feasible set of PnP is the intersection of the feasible set of camera resectioning and the rotation constraints, for each of which ACQ is shown to hold by \citet{C21} and \citet{WKC26}, respectively. We show that the two meet non-tangentially such that ACQ holds for the intersection. The first condition is established in Theorem~\ref{thm:ACQ}.

The second and third conditions are again regularity conditions on $\mathcal X_\theta$ and on the image of $\mathcal X_{\bar\theta}$ under the mapping $x\mapsto \bar Sx$, respectively, and we show that both hold in Theorem~\ref{thm:smoothness} and Theorem~\ref{thm:non-branch-point}.
The last condition requires finding a multiplier such that the restricted Slater condition holds, and we construct it explicitly in Theorem~\ref{thm:restricted-Slater}.
Altogether, by Lemma~\ref{lem:Cifuentes}, we have Theorem~\ref{thm:stable}.

\begin{definition}\label{def:maps:manifold}
Let $\mathcal M$ and $\mathcal N$ be smooth manifolds, and consider the mapping $\Phi:\mathcal M\to\mathcal N$. Then
\begin{enumerate}[(i)]
  \item $\Phi$ is a \textit{homeomorphism} if it is a bijection, and both $\Phi$ and $\Phi^{-1}$ are continuous.
  \item $\Phi$ is a \textit{diffeomorphism}\footnote{A diffeomorphism is in particular an embedding.} if it is a bijection, and both $\Phi$ and $\Phi^{-1}$ are smooth.
  \item $\Phi$ is an \textit{immersion} if it is smooth, and its differential $d\Phi_x:T_x(\mathcal M)\to T_{\Phi(x)}(\mathcal N)$ is injective for all $x\in\mathcal M$.
  \item $\Phi$ is an \textit{embedding} if it is an immersion, and a homeomorphism onto its image $\Phi(\mathcal M)$.
\end{enumerate}
\end{definition}

\begin{lemma}[{\citet[Proposition 5.2]{L13}}]\label{lem:embedding-image}
If $\Phi:\mathcal M\to\mathcal N$ is an embedding, then $\Phi(\mathcal M)$ is
an embedded submanifold of $\mathcal N$, and $\Phi$ is a diffeomorphism onto
$\Phi(\mathcal M)$. In particular,
$\dim(\Phi(\mathcal M))=\dim(\mathcal M)$.
\end{lemma}

\subsection{Abadie Constraint Qualification}\label{appx:ACQ}
\begin{definition}[ACQ]\label{def:ACQ}
Let $\mathcal X=\{x\mid f(x)=0\}$. The \textit{Abadie constraint qualification} holds at $x\in\mathcal X$ if $\mathcal X$ is a smooth manifold nearby $x$ and $T_x(\mathcal X)=\ker(\mathbf J_f(x))$.
\end{definition}

\begin{definition}[Non-tangentiality]\label{def:non-tangential}
Two manifolds $\mathcal M_1$ and $\mathcal M_2$ \textit{meet non-tangentially} at a point $x\in\mathcal M_1\cap\mathcal M_2$ if $\mathcal M_1\cap\mathcal M_2$ is a manifold near $x$ and $T_x(\mathcal M_1\cap\mathcal M_2)=T_x(\mathcal M_1)\cap T_x(\mathcal M_2)$.
\end{definition}

The feasible set of PnP is the intersection of two sets: (\romannumeral1) the feasible set of camera resectioning, and (\romannumeral2) the set of camera poses, determined by the rotation constraints. ACQ follows for both sets~\citep{C21,WKC26}, but is in general not preserved under intersection. The following Lemma characterizes when it is.

\begin{lemma}[{\citet[Lemma 4.3]{C21}}]\label{lemma:non-tangential}
Let $\mathcal{X}_1=\{x\mid f_1(x)=0\}$, $\mathcal{X}_2=\{x\mid f_2(x)=0\}$, and $\mathcal{X}=\{x\mid f_1(x)=0,f_2(x)=0\}$. If ACQ holds for both $\mathcal X_1,\mathcal X_2$ at $\bar x\in\mathcal{X}$, then ACQ holds for $\mathcal X$ at $\bar x$ if and only if $\mathcal X_1,\mathcal X_2$ meet non-tangentially at $\bar x$.
\end{lemma}

In practice, we can verify non-tangentiality by comparing the dimensions.
\begin{lemma}\label{lemma:tangent-intersect}
Two manifolds $\mathcal M_1$ and $\mathcal M_2$ meet non-tangentially at a point $x$ if $\mathcal M_1\cap\mathcal M_2$ is a manifold near $x$ and $\dim(\mathcal M_1\cap\mathcal M_2)=\dim(T_x(\mathcal M_1)\cap T_x(\mathcal M_2))$.
\end{lemma}

\begin{proof}
Since $\mathcal M_1\cap\mathcal M_2$ is a manifold near $x$, its tangent space $T_x(\mathcal M_1\cap\mathcal M_2)$ has dimension $\dim(\mathcal M_1\cap\mathcal M_2)$. Furthermore, it is a submanifold of both $\mathcal M_1$ and $\mathcal M_2$, hence $T_x(\mathcal M_1\cap\mathcal M_2)\subseteq T_x(\mathcal M_1)\cap T_x(\mathcal M_2)$. Both are linear subspaces of the same dimension, and are therefore equal. By Definition~\ref{def:non-tangential}, they meet non-tangentially at $x$.
\end{proof}

Recall the feasible sets $\mathcal W_\theta$ of Problem~\ref{prob:naive} in \eqref{eq:feasible-set:naive} and $\mathcal X_\theta$ of Problem~\ref{prob:lifted} in \eqref{eq:feasible-set:lifted}. By Corollary~\ref{cor:naive} and~\ref{cor:lifted}, Problem~\ref{prob:naive} and~\ref{prob:lifted} have the unique equivalence classes of global minimizer $[\bar w]$ and $[\bar x]$, where $\bar w=(1,z_\mathrm{gt},0)\in\mathcal W_{\bar\theta}$ and $\bar x=(\bar h,\bar z,0)\in\mathcal X_{\bar\theta}$ with $\bar h=1/\|z_\mathrm{gt}\|$ and $\bar z=z_\mathrm{gt}/\|z_\mathrm{gt}\|$. Consider the set
\begin{equation}\label{eq:W-normalized}
\begin{aligned}
\mathcal W'_\theta\triangleq\{w\in\mathbb R^{13+2N}\mid\ &m_{\theta,i}(z,v)=0,\ i\in[N],\\
&r_j(h,z)=0,\ j\in[20],\\
&\|z\|^2=1\},
\end{aligned}
\end{equation}
which is $\mathcal W_\theta$ with $h^2=1$ replaced by $\|z\|^2=1$, and $\bar w'\triangleq(\bar h,\bar z,0)\in\mathcal W'_{\bar\theta}$. Note that $(h,z,v)\mapsto(h/\|z\|,z/\|z\|,v)$ maps $\mathcal W_\theta$ into $\mathcal W'_\theta$ and $\bar w$ to $\bar w'$, and $(h,z,v)\mapsto(h,z,v\otimes z)$ maps $\mathcal W'_\theta$ onto $\mathcal X_\theta$ and $\bar w'$ to $\bar x$. 
In the proof of Theorem~\ref{thm:ACQ}, we first show ACQ holds at $\bar w'\in\mathcal W'_{\bar\theta}$, then extend it $\bar x\in\mathcal X_{\bar\theta}$.

Recall the affine mapping $A:\mathbb{R}^{2N}\to\mathbb{R}^{2N\times12}$ in~\eqref{eq:affine}. For notation simplicity, we denote $d_i(z)\triangleq e_3^\top P_iz$ and
\begin{equation}\label{eq:depth}
D(z)\triangleq\begin{pmatrix}
    \ddots&&\\&d_i(z)I_2&\\&&\ddots
\end{pmatrix}\in\mathbb{R}^{2N\times2N}.
\end{equation}
Since $\bar z=z_\mathrm{gt}/\|z_\mathrm{gt}\|$, we have $d_i(\bar z)\neq0$ for all $i\in[N]$, and $\det(D(\bar z))=\prod_{i=1}^Nd_i(\bar z)^2\neq0$, hence $D(\bar z)$ is invertible. By continuity, $d_i(z)\neq0$ for all $i\in[N]$ and $D(z)$ is invertible if $z$ is near $\bar z$. 
Since ACQ and non-tangentiality are local properties, the remainder of this section considers only locally when $z$ is near $\bar z$.
For such $z$, the $i$-th block of $A(\theta)z$ is $\theta_id_i(z)-\begin{bmatrix}I_2&0\end{bmatrix}KP_iz$, and dividing by $d_i(z)\neq0$ gives
\begin{equation}\label{eq:affine-projection}
  A(\theta)z=0\iff\theta=\pi(z),
\end{equation}
where $\pi$ is the projection in \eqref{eq:projection}.
The following Lemma is required for the camera resectioning part~\citep[Assumption 1.1]{C21} in the proof of Theorem~\ref{thm:ACQ}.

\begin{lemma}\label{lemma:ACQ:fix-rank}
Let $\mathcal M_\mathrm{aff}\triangleq\mathrm{Im}(A)$ and $\mathcal M_\mathrm{rank}\triangleq\{X\in\mathbb{R}^{2N\times 12}\mid\mathrm{rank}(X)=11\}$. Then $\mathcal M_\mathrm{aff}$ and $\mathcal M_\mathrm{rank}$ meet non-tangentially at $A(\bar\theta)$.
\end{lemma}

\begin{proof}
Note that $A(\bar\theta)\in\mathcal M_\mathrm{aff}\cap\mathcal M_\mathrm{rank}$ by the proof of Proposition~\ref{prop:mle:unique}. 
We claim that (\romannumeral1) $\mathcal M_\mathrm{aff}\cap\mathcal M_\mathrm{rank}$ is an 11-dimensional manifold near $A(\bar\theta)$, and (\romannumeral2) $T_{A(\bar\theta)}(\mathcal M_\mathrm{aff})\cap T_{A(\bar\theta)}(\mathcal M_\mathrm{rank})$ is also 11-dimensional.

\begin{quote}
\small\itshape
``We show (\romannumeral1) by constructing an embedding $\Phi$ of an 11-dimensional manifold $U$ whose image is $\mathcal M_\mathrm{aff}\cap\mathcal M_\mathrm{rank}$ near $A(\bar\theta)$.''
\end{quote}

Every $M\in\mathcal M_\mathrm{aff}\cap\mathcal M_\mathrm{rank}$ takes the form $M=A(\theta)$ with some nonzero null vector $z$, and if $d_i(z)\neq0$ for all $i\in[N]$, then by \eqref{eq:affine-projection}, $M=A(\pi(z))$. This suggests a mapping from $z$ to $A(\pi(z))$. However, $z$ is determined only up to scale since $\pi$ is scale-invariant, we first fix the scale by the affine hyperplane
$$
H\triangleq\{z\in\mathbb{R}^{12}\mid \bar z^\top z=1\}.
$$
Note that $H$ is an 11-dimensional affine subspace containing $\bar z$. Let
\begin{align*}
U\triangleq\{z\in H&\mid d_i(z)\neq0,\ i\in[N],\\
&\quad\mathrm{rank}(A(\pi(z)))=11\}.
\end{align*}
We have $\bar z\in U$ since $d_i(\bar z)\neq0$ for all $i\in[N]$ and $\mathrm{rank}(A(\pi(\bar z)))=\mathrm{rank}(A(\bar\theta))=11$.

We claim that $U$ is open in $H$. Since $d_i(\cdot)$ is continuous, the first condition $d_i(z)\neq0$ for all $i\in[N]$ defines an open set. For the second condition, $\mathrm{rank}(A(\pi(z)))=11$ implies some $11\times11$ minor of $A(\pi(z))$ is nonzero. Since $A(\pi(\cdot))$ is also continuous, we have $\mathrm{rank}(A(\pi(z)))\geq11$ nearby. Note that the first condition also implies $A(\pi(z))z=0$ by \eqref{eq:affine-projection}, and so $\mathrm{rank}(A(\pi(z)))\leq11$. Altogether, $\mathrm{rank}(A(\pi(z)))=11$ nearby, the second condition also defines an open set. 
Finally, $U$ is the intersection of two open sets, and hence open in $H$, as claimed.
It follows from \citep[Proposition 3.23]{B23} that $U$ is an 11-dimensional manifold.

To this end, we denote the mapping
$$
\Phi:U\to\mathbb{R}^{2N\times12},\quad z\mapsto A(\pi(z)),
$$
and show that $\Phi$ is an embedding. We first claim that $\Phi$ is an immersion.
It is smooth, since $\pi$ is rational with nonzero denominators on $U$ and $A$ is affine. The Jacobian of $\pi$ at $z$ is
\begin{equation}\label{eq:Jacobian:pi-z}
\mathbf J_\pi(z)=-D(z)^{-1}A(\pi(z))\in\mathbb{R}^{2N\times 12}
\end{equation}
since
\begin{align*}
\mathbf J_{\pi_i}(z)&=\frac{(e_3^\top P_iz)\begin{bmatrix}I_2&0\end{bmatrix}KP_i-\begin{bmatrix}I_2&0\end{bmatrix}KP_iz(e_3^\top P_i)}{(e_3^\top P_iz)^2}\\
&=\frac{\begin{bmatrix}I_2&0\end{bmatrix}KP_i-\pi_i(z)(e_3^\top {P_i})}{d_i(z)}=-\frac{[A(\pi(z))]_i}{d_i(z)}.
\end{align*}
Since $A$ is affine, by chain rule and \eqref{eq:affine}, the differential of $\Phi$ at $z\in U$ is
$$
d\Phi_z:T_z(U)\to\mathbb{R}^{2N\times12},\quad v\mapsto\begin{pmatrix}\vdots\\ [\mathbf J_\pi(z)v]_i\,e_3^\top P_i\\ \vdots\end{pmatrix}.
$$

To show that $d\Phi_z$ is injective, let $v_1,v_2\in T_z(U)$ with $d\Phi_z(v_1)=d\Phi_z(v_2)$. Then we have
\begin{align*}
&[\mathbf J_\pi(z)(v_1-v_2)]_i\,e_3^\top P_i=0,\quad i\in[N]\\
\implies\ &\mathbf J_\pi(z)(v_1-v_2)=0\\
\implies\ &A(\pi(z))(v_1-v_2)=0
\end{align*}
The first implication holds because $e_3^\top P_i\neq0$ for all $i\in[N]$, the second holds by \eqref{eq:Jacobian:pi-z} and $D(z)$ is invertible for $z\in U$.
Since $A(\pi(z))z=0$ by \eqref{eq:affine-projection} and $\mathrm{rank}(A(\pi(z)))=11$ for $z\in U$, we have 
$$
v_1-v_2\in\ker(A(\pi(z)))=\mathrm{span}\{z\}.
$$
On the other hand, $T_z(U)=T_z(H)=\bar z^\perp$~\citep[Proposition 3.23]{B23},
and so $v_1-v_2\in\bar z^\perp$.
Let $v_1-v_2=cz$ for some scalar $c$, we have 
$$
0=\bar z^\top(v_1-v_2)=c\,\bar z^\top z=c
$$
since $z\in H$. That is, $v_1=v_2$ and $d\Phi_z$ is injective, which implies $\Phi$ is an immersion, as claimed.

Since $\Phi$ is an immersion, by the local embedding theorem~\citep[Theorem 4.25]{L13}, there exists an open neighborhood $U'\subseteq U$ of $\bar z$ such that the mapping
\begin{equation*}
  \Phi':U'\to\mathbb{R}^{2N\times12},\quad z\mapsto A(\pi(z))
\end{equation*}
is an embedding. Since $U'$ is still open in $H$ and contains $\bar z$, all properties of $U$ established above hold for $U'$. For the sake of simplicity, we denote $U'$ by $U$ and $\Phi'$ by $\Phi$ in the remainder of the proof.

Finally, we show that the image $\Phi(U)$ is $\mathcal M_\mathrm{aff}\cap\mathcal M_\mathrm{rank}$ near $A(\bar\theta)$. By construction, $\Phi(U)\subseteq\mathcal M_\mathrm{aff}\cap\mathcal M_\mathrm{rank}$. For the converse, let $M\in\mathcal M_\mathrm{aff}\cap\mathcal M_\mathrm{rank}$ be sufficiently close to $A(\bar\theta)$. We claim that $M$ has a null vector in $U$. Since $\mathrm{rank}(M)=11$, $M$ has a unit null vector $u$. Let $b\triangleq u-(\bar z^\top u)\bar z$ be the orthogonal component to $\bar z$. Since $\ker(A(\bar\theta))=\mathrm{span}\{\bar z\}$ and $\bar z^\top b=0$, we have $\|A(\bar\theta)b\|\geq\sigma\|b\|$ where $\sigma>0$ is the smallest nonzero singular value of $A(\bar\theta)$. Then
\begin{align*}
\sigma\|b\|\leq\|A(\bar\theta)b\|&=\|A(\bar\theta)u\|\\
&=\|(A(\bar\theta)-M)u\|\leq\|A(\bar\theta)-M\|,
\end{align*}
where the first equality holds because $A(\bar\theta)\bar z=0$ and the second holds by $Mu=0$, and the last inequality uses $\|u\|=1$. Since $M$ is sufficiently close to $A(\bar\theta)$, $b$ is also sufficiently small.
Since $\|b\|^2=1-(\bar z^\top u)^2$, we have $\bar z^\top u\neq0$. Let $z\triangleq u/(\bar z^\top u)$, then $Mz=0$ and $\bar z^\top z=1$, i.e., $z\in H$. Furthermore, 
$$
z-\bar z=\frac{u}{\bar z^\top u}-\frac{(\bar z^\top u)\bar z}{\bar z^\top u}=\frac{b}{\bar z^\top u}
$$
is sufficiently small, that is, $z\in U$ is a null vector, as claimed. Since $M\in\mathcal M_\mathrm{aff}$, we have $M=A(\theta)$ for some $\theta$. Since $z\in U$, we have $d_i(z)\neq0$ for all $i$, and so $A(\theta)z=0$ implies $\theta=\pi(z)$ by~\eqref{eq:affine-projection}, i.e., $M=\Phi(z)$. Therefore, $\Phi(U)$ is $\mathcal M_\mathrm{aff}\cap\mathcal M_\mathrm{rank}$ near $A(\bar\theta)$.

Altogether, since $\Phi$ is an embedding and $U$ is $11$-dimensional, $\mathcal M_\mathrm{aff}\cap\mathcal M_\mathrm{rank}$ is an $11$-dimensional manifold near $A(\bar\theta)$ by Lemma~\ref{lem:embedding-image}, which proves (\romannumeral1).

\begin{quote}
\small\itshape
``We show (\romannumeral2) by writing out both tangent spaces at $A(\bar\theta)$ and computing their intersection.''
\end{quote}

Consider the linear mapping
$$
\phi:\mathbb R^{2N}\to\mathbb R^{2N\times12},\quad \theta\mapsto\begin{pmatrix}\vdots\\ \theta_ie_3^\top P_i\\ \vdots\end{pmatrix}.
$$
By~\eqref{eq:affine}, every $\Gamma\in T_{A(\bar\theta)}(\mathcal M_\mathrm{aff})$ takes the form $\Gamma=\phi(\theta^\Gamma)$ for some $\theta^\Gamma\in\mathbb R^{2N}$. We claim that $\phi$ is injective. Let $\theta_1,\theta_2\in\mathbb R^{2N}$ with $\phi(\theta_1)=\phi(\theta_2)$. Then $(\theta_{1,i}-\theta_{2,i})e_3^\top P_i=0$ for all $i\in[N]$. Since $e_3^\top P_i\neq0$, we have $\theta_1=\theta_2$.

Recall that $\bar z$ forms a basis for the right null space of $A(\bar\theta)$, and let $Y\in\mathbb R^{2N\times(2N-11)}$ be a matrix whose columns form a basis for the left null space. For $\Gamma$ to be in $T_{A(\bar\theta)}(\mathcal M_\mathrm{rank})$, it must satisfy $Y^\top\Gamma\bar z=Y^\top D(\bar z)\theta^\Gamma=0$~\citep[Section~7.5]{B23}. Therefore, $T_{A(\bar\theta)}(\mathcal M_\mathrm{aff})\cap T_{A(\bar\theta)}(\mathcal M_\mathrm{rank})=\phi(\ker(Y^\top D(\bar z)))$.
Since $D(\bar z)$ is invertible, we have $\mathrm{rank}(Y^\top D(\bar z))=\mathrm{rank}(Y^\top)=2N-11$. By rank-nullity theorem, $\dim(\ker(Y^\top D(\bar z)))=2N-(2N-11)=11$. Since $\phi$ is injective, the intersection is also 11-dimensional, which proves (\romannumeral2).

Altogether, by Lemma~\ref{lemma:tangent-intersect}, $\mathcal M_\mathrm{aff}$ and $\mathcal M_\mathrm{rank}$ meet non-tangentially at $A(\bar\theta)$.
\end{proof}

Recall $\bar w'=(\bar h,\bar z,0)\in\mathcal W'_{\bar\theta}$ and $\bar x=(\bar h,\bar z,0)\in\mathcal X_{\bar\theta}$ with $\bar h=1/\|z_\mathrm{gt}\|$ and $\bar z=z_\mathrm{gt}/\|z_\mathrm{gt}\|$.
We now show the Abadie constraint qualification.

\begin{theorem}\label{thm:ACQ}
ACQ holds at $\bar x\in\mathcal X_{\bar\theta}$.
\end{theorem}

\begin{proof}
We first show that ACQ holds at $\bar w'\in\mathcal W'_{\bar\theta}$, then extend it to $\bar x\in\mathcal{X}_{\bar\theta}$. 
Consider the sets
\begin{align*}
\mathcal W_1&\triangleq\{w\mid m_{\bar\theta,i}(z,v)=0,\ i\in[N],\ \|z\|^2=1\},\\
\mathcal W_2&\triangleq\{w\mid r_j(h,z)=0,\ j\in[20],\ \|z\|^2=1\},
\end{align*}
where $w=(h,z,v)$ such that $\mathcal W'_{\bar\theta}=\mathcal W_1\cap\mathcal W_2$.
Note that $\mathcal W_1$ is the feasible set of nearest rank deficient matrix problem analyzed by \citet{C21}. By Lemma~\ref{lemma:ACQ:fix-rank}, it follows from \citet[Lemma 4.7]{C21} that ACQ holds at $\bar w'\in\mathcal W_1$. 
Note that ACQ holds for the 20 redundant rotation constraints on $\mathrm{SO}(3)$ by~\citet[Theorem 6]{WKC26}, and is preserved when introducing the homogeneous variable $h$ with $h^2=1$ by~\citet[Section 4.3]{CAP22}.
$\mathcal W'_2$ differs from this only by the variables $t,v$\footnote{$w=(h,z,v)=(h,\text{vec}(R),t,v)$.}, and the constraint $\|z\|^2=1$ that replaces $h^2=1$, neither of which affects ACQ\footnote{By \citet[Definition 3.1]{CAP22}, ACQ holds if the rank of the Jacobian of the constraints equals the codimension of the variety. Introducing (free) variables $t,v$ adds zero columns to the Jacobian and dimensions to the variety. Since the gradients of $h^2$ and $\|z\|^2$ are not in the span of the gradients of the rotation constraints, both $h^2=1$ and $\|z\|^2=1$ increase the rank by one and reduce the dimension by one. Hence replacing $h^2=1$ by $\|z\|^2=1$ leaves both unchanged.}, hence ACQ holds at $\bar w'\in\mathcal W_2$.

We proceed to show the conditions of Lemma~\ref{lemma:tangent-intersect}: (\romannumeral1) $\mathcal W_1\cap\mathcal W_2$ is a 6-dimensional manifold near $\bar w'$, and (\romannumeral2) $T_{\bar w'}(\mathcal W_1)\cap T_{\bar w'}(\mathcal W_2)$ is also 6-dimensional.

For (\romannumeral1), recall $\mathcal Z$ in \eqref{eq:feasible-set:mle} and consider the mapping
$$
\Phi:\mathcal Z\to\mathbb R^{13+2N},\quad z\mapsto(1/\|z\|,z/\|z\|,\bar\theta-\pi(z)).
$$
By the proof of Theorem~\ref{thm:naive} ($\phi$ in \eqref{eq:thm1:phi} at $\bar\theta$) and the proof of Theorem~\ref{thm:lifted} (rescaling by $1/\|z\|$), $\Phi$ is well-defined and $\Phi(\mathcal Z)\subseteq\mathcal W'_{\bar\theta}$.
Since $\mathcal Z$ is an open subset of the 6-dimensional manifold $\mathrm{SO}(3)\times\mathbb R^3$, it is a 6-dimensional manifold~\citep[Proposition~3.23]{B23}. 

We first show that $\Phi$ is an embedding. It is smooth since $\|z\|\neq0$ and $\pi$ is rational with nonzero denominators on $\mathcal Z$. 
Next, consider the mapping
$$
\rho:\{(h,z,v)\mid h\neq0\}\to\mathbb R^{12},\quad (h,z,v)\mapsto z/h,
$$
which is smooth since $h\neq0$. Since the first component of $\Phi(z)$ is nonzero, $\Phi(z)$ lies in the domain of $\rho$, and
\begin{equation}\label{eq:left-inverse}
\rho(\Phi(z))=\frac{z/\|z\|}{1/\|z\|}=z\quad\text{for all }z\in\mathcal Z.
\end{equation}
Differentiating~\eqref{eq:left-inverse} gives $d\rho_{\Phi(z)}\circ d\Phi_z=I$ (identity map) by chain rule, hence $d\Phi_z$ is injective for all $z\in\mathcal Z$, i.e., $\Phi$ is an immersion.

We next show that $\Phi$ is a homeomorphism onto its image. Let $z_1,z_2\in\mathcal Z$ with $\Phi(z_1)=\Phi(z_2)$, then $z_1=\rho(\Phi(z_1))=\rho(\Phi(z_2))=z_2$ by~\eqref{eq:left-inverse}, and so $\Phi$ is a bijection onto $\Phi(\mathcal Z)$.
By~\eqref{eq:left-inverse}, its inverse
$$
\Phi^{-1}:\Phi(\mathcal Z)\to\mathcal Z,\quad w\mapsto\rho(w)
$$ 
is continuous since $\rho$ is smooth, thus $\Phi$ is a homeomorphism onto its image. Altogether, $\Phi$ is an embedding, and $\Phi(\mathcal Z)$ is a 6-dimensional manifold by Lemma~\ref{lem:embedding-image}.

Finally, we show that $\Phi(\mathcal Z)$ is $\mathcal W'_{\bar\theta}$ near $\bar w'$. Since $\Phi(\mathcal Z)\subseteq\mathcal W'_{\bar\theta}$, it remains to show the converse. 
Let $w=(h,z,v)\in\mathcal W'_{\bar\theta}$ be sufficiently close to $\bar w'$. To show $w\in\Phi(\mathcal Z)$, let $z'\triangleq z/h$. Since $\bar h>0$ and $e_3^\top P_i\bar z\neq0$, by continuity we have $h>0$ and $e_3^\top P_iz\neq0$ for all $i\in[N]$. We first show $z'\in\mathcal Z$. Since $r_j(1,z')=r_j(h,z)/h^2=0$, the rotation component of $z'$ lies in $\mathrm{SO}(3)$ by Appendix~\ref{appx:rot}, and $e_3^\top P_iz'=e_3^\top P_iz/h\neq0$, hence $z'\in\mathcal Z$. We next show $\Phi(z')=w$. Since $\|z\|=1$ and $h>0$, we have $\|z'\|=1/h$, and so the first two components of $\Phi(z')$ are $(h,z)$. Moreover, dividing $m_{\bar\theta,i}(z,v)=0$ by $e_3^\top P_iz\neq0$ gives $v_i=\bar\theta_i-\pi_i(z)=\bar\theta_i-\pi_i(z')$, since $\pi_i$ is scale-invariant. Hence $\Phi(z')=(h,z,v)=w$, i.e., $w\in\Phi(\mathcal Z)$.
Therefore, $\Phi(\mathcal Z)$ is $\mathcal W'_{\bar\theta}$ near $\bar w'$, and $\mathcal W_1\cap\mathcal W_2=\mathcal W'_{\bar\theta}$ is a 6-dimensional manifold near $\bar w'$, which proves (\romannumeral1).

For (\romannumeral2), recall that since ACQ holds at $\bar w'$ for both $\mathcal W_1$ and $\mathcal W_2$, their tangent spaces are the kernels of the Jacobian of their constraints by Definition~\ref{def:ACQ}. Let 
$$
\mathcal S\triangleq\{(h,z)\mid r_j(h,z)=0,\ j\in[20],\ \|z\|^2=1\}.
$$ 
such that $\mathcal W_2=\mathcal S\times\mathbb R^{2N}$ since the constraints of $\mathcal W_2$ involve only $(h,z)$. By \citet[Proposition~3.20]{B23}, we have $T_{\bar w'}(\mathcal W_2)=T_{(\bar h,\bar z)}(\mathcal S)\times\mathbb R^{2N}$. By (\romannumeral1), $\mathcal S$ is the image of $\mathrm{SO}(3)\times\mathbb R^3$ under the embedding $z\mapsto(1/\|z\|,z/\|z\|)$ near $(\bar h,\bar z)$, and so $T_{(\bar h,\bar z)}(\mathcal S)$ is also 6-dimensional. For $\mathcal W_1$, differentiating the model constraints~\eqref{eq:model} at $\bar w'$ shows that every $\gamma=(\gamma_h,\gamma_z,\gamma_v)\in T_{\bar w'}(\mathcal W_1)$ satisfies
$$
[A(\bar\theta)]_i\gamma_z-d_i(\bar z)\gamma_{v,i}=0,\quad i\in[N],
$$
where $\gamma_{v,i}\in\mathbb R^2$ is the $i$-th block of $\gamma_v$, i.e., $A(\bar\theta)\gamma_z=D(\bar z)\gamma_v$ with $D$ in~\eqref{eq:depth}, and the normalization constraint gives $\bar z^\top\gamma_z=0$.
Consider the linear mapping
$$
L:T_{\bar w'}(\mathcal W_1)\cap T_{\bar w'}(\mathcal W_2)\to T_{(\bar h,\bar z)}(\mathcal S),\quad \gamma\mapsto(\gamma_h,\gamma_z).
$$
It is trivially well-defined. It is injective because $L(\gamma_1)=L(\gamma_2)$ gives $(\gamma_{1,h},\gamma_{1,z})=(\gamma_{2,h},\gamma_{2,z})$, and $\gamma_{1,v}=D(\bar z)^{-1}A(\bar\theta)\gamma_{1,z}=D(\bar z)^{-1}A(\bar\theta)\gamma_{2,z}=\gamma_{2,v}$ since $D(\bar z)$ is invertible, which implies $\gamma_1=\gamma_2$.
For surjectivity, let $(\gamma_h,\gamma_z)\in T_{(\bar h,\bar z)}(\mathcal S)$ and $\gamma_v\triangleq D(\bar z)^{-1}A(\bar\theta)\gamma_z$. Then $\gamma\in T_{\bar w'}(\mathcal W_2)$ by construction, and $\gamma\in T_{\bar w'}(\mathcal W_1)$ since the model constraints hold by the choice of $\gamma_v$ and $\bar z^\top\gamma_z=0$ holds because $\|z\|^2=1$ is also a constraint of $\mathcal W_2$. 
Altogether, $L$ is bijective and since it is linear, $\dim(T_{\bar w'}(\mathcal W_1)\cap T_{\bar w'}(\mathcal W_2))=\dim(T_{(\bar h,\bar z)}(\mathcal S))=6$, which proves (\romannumeral2).

Finally, by Lemma~\ref{lemma:tangent-intersect}, $\mathcal W_1$ and $\mathcal W_2$ meet non-tangentially at $\bar w'$, and by Lemma~\ref{lemma:non-tangential}, ACQ holds at $\bar w'\in\mathcal W_1\cap\mathcal W_2=\mathcal W'_{\bar\theta}$.
It follows from \citet[Lemma 4.6]{C21} that ACQ holds at $\bar x\in\mathcal X_{\bar\theta}$ because $\mathcal X_{\bar\theta}$ is the image of $\mathcal W'_{\bar\theta}$ under the embedding $(h,z,v)\mapsto(h,z,v\otimes z)$, and the ACQ property is preserved if the ideal generated by the algebraic description remains unchanged.
\end{proof}

\subsection{Smoothness}\label{appx:smoothness}
Recall $\bar w'=(\bar h,\bar z,0)\in\mathcal W'_{\bar\theta}$ and $\bar x=(\bar h,\bar z,0)\in\mathcal X_{\bar\theta}$ with $\bar h=1/\|z_\mathrm{gt}\|$ and $\bar z=z_\mathrm{gt}/\|z_\mathrm{gt}\|$.

\begin{definition}[{\citet[Definition~4.1]{CAP22}}]\label{def:smooth}
The mapping $\theta\mapsto\mathcal X_\theta$ is \emph{smooth} nearby $(\bar\theta,\bar x)$ if $\{(\theta,x)\mid x\in\mathcal X_\theta\}$ is a smooth manifold nearby $(\bar\theta,\bar x)$ whose dimension is $2N+\dim(\mathcal X_{\bar\theta})$, where $\dim(\mathcal X_{\bar\theta})$ is the dimension of $\mathcal X_{\bar\theta}$ near $\bar x$.
\end{definition}

\begin{theorem}\label{thm:smoothness}
The mapping $\theta\mapsto\mathcal X_\theta$ is smooth nearby $(\bar\theta,\bar x)$.
\end{theorem}

\begin{proof}
Consider the sets
\begin{align*}
\mathcal M_\mathcal W&\triangleq\{(\theta,w)\mid w\in\mathcal W'_\theta\}\subseteq\mathbb R^{2N}\times\mathbb R^{13+2N},\\
\mathcal M_\mathcal X&\triangleq\{(\theta,x)\mid x\in\mathcal X_\theta\}\subseteq\mathbb R^{2N}\times\mathbb R^{13+24N},
\end{align*}
where $\mathcal W'_\theta$ and $\mathcal X_\theta$ are in \eqref{eq:W-normalized} and \eqref{eq:feasible-set:lifted}, respectively.
We first show that $\mathcal M_\mathcal W$ is a manifold nearby $(\bar\theta,\bar w')$, then extend it to $\mathcal M_\mathcal X$. 
Since $\|z\|^2=1$ does not involve $v$, \eqref{eq:set-translation} still holds for $\mathcal W'_\theta$, that is, 
\begin{equation}\label{eq:set-translation-prime}
\mathcal W'_\theta=\mathcal W'_{\bar\theta}+(0,0,\theta-\bar\theta).
\end{equation}
Consider the mapping
\begin{align*}
    \Phi:\mathbb R^{2N}\times\mathbb R^{13+2N}&\to\mathbb R^{2N}\times\mathbb R^{13+2N}\\
    (\theta,w)&\mapsto(\theta,w+(0,0,\theta-\bar\theta))
\end{align*}
such that $\Phi(\mathbb R^{2N}\times\mathcal W'_{\bar\theta})=\mathcal M_\mathcal W$ by \eqref{eq:set-translation-prime}.
Since $\mathcal W'_{\bar\theta}$ is a 6-dimensional manifold near $\bar w'$ by the proof of Theorem~\ref{thm:ACQ}, $\mathbb R^{2N}\times\mathcal W'_{\bar\theta}$ is a $(2N+6)$-dimensional manifold near $(\bar\theta,\bar w')$~\citep[Proposition~3.20]{B23}. $\Phi$ is a diffeomorphism\footnote{It is smooth, trivially a bijection, and the inverse $(\theta,w)\mapsto(\theta,w-(0,0,\theta-\bar\theta))$ is also smooth.}, which is also an embedding, hence $\mathcal M_\mathcal W$ is a $(2N+6)$-dimensional manifold near $(\bar\theta,\bar w')$ by Lemma~\ref{lem:embedding-image}.

Lastly, it follows from \citet[Section~5.2]{C21} that $\mathcal M_\mathcal X$ is a smooth manifold near $(\bar\theta,\bar x)$ since it is the image of $\mathcal M_\mathcal W$ under the embedding $(\theta,h,z,v)\mapsto(\theta,h,z,v\otimes z)$, and is a $(2N+6)$-dimensional manifold by Lemma~\ref{lem:embedding-image}. Likewise, $\mathcal X_{\bar\theta}$ is a 6-dimensional manifold near $\bar x$ by the proof of Theorem~\ref{thm:ACQ}. Altogether, $\dim(\mathcal M_\mathcal X)=2N+\dim(\mathcal X_{\bar\theta})$, and $\theta\mapsto\mathcal X_\theta$ is smooth nearby $(\bar\theta,\bar x)$ by Definition~\ref{def:smooth}.
\end{proof}

\subsection{Non-branch point}\label{appx:non-branch-point}

Recall the dual slack variable $\bar S$ in \eqref{eq:slack}, and $\bar x=(\bar h,\bar z,0)\in\mathcal X_{\bar\theta}$ with $\bar h=1/\|z_\mathrm{gt}\|$ and $\bar z=z_\mathrm{gt}/\|z_\mathrm{gt}\|$.

\begin{definition}[{\citet[Definition~4.3]{CAP22}}]\label{def:branch}
We say $\bar x$ is not a \textit{branch point} of $\mathcal X_{\bar\theta}$ with respect to $x\mapsto\bar Sx$ if $T_{\bar x}(\mathcal X_{\bar\theta})\cap\ker(\bar S)=\{0\}$.
\end{definition}

For simplicity, let $f_{\bar\theta}=(f_{\bar\theta,1},\dots,f_{\bar\theta,\bar\ell})$ collect all the constraints of $\mathcal X_{\bar\theta}$ in~\eqref{eq:feasible-set:lifted} such that
\begin{equation}\label{eq:constraint-f}
\mathcal X_{\bar\theta}=\{x\mid f_{\bar\theta}(x)=0\}.
\end{equation}
Note that $\ker(\bar S)=\{x\mid y=0\}$, hence only $h$ and $z$ are involved. Let $\mathbf h=(h,z)$, we first derive a special case for our problem.

\begin{lemma}\label{lemma:non-branch-point}
$\bar x$ is not a branch point of $\mathcal X_{\bar\theta}$ with respect to $x\mapsto\bar Sx$ if and only if $\ker(\mathbf J_{f_{\bar\theta},\mathbf h}(\bar x))$ is trivial.
\end{lemma}

\begin{proof}
Write $\gamma=(\gamma_{\mathbf h},\gamma_y)$ with $\gamma_{\mathbf h}=(\gamma_h,\gamma_z)$. By Theorem~\ref{thm:ACQ} and Definition~\ref{def:ACQ}, $T_{\bar x}(\mathcal X_{\bar\theta})=\ker(\mathbf J_{f_{\bar\theta}}(\bar x))$, and $\ker(\bar S)=\{\gamma\mid\gamma_y=0\}$. Hence $\gamma\in T_{\bar x}(\mathcal X_{\bar\theta})\cap\ker(\bar S)$ if and only if $\gamma_y=0$ and
\begin{align*}
\mathbf J_{f_{\bar\theta}}(\bar x)\gamma&=\mathbf J_{f_{\bar\theta},\mathbf h}(\bar x)\gamma_{\mathbf h}+\mathbf J_{f_{\bar\theta},y}(\bar x)\gamma_y\\
&=\mathbf J_{f_{\bar\theta},\mathbf h}(\bar x)\gamma_{\mathbf h}=0.
\end{align*}
That is, $T_{\bar x}(\mathcal X_{\bar\theta})\cap\ker(\bar S)=\ker(\mathbf J_{f_{\bar\theta},\mathbf h}(\bar x))\times\{0\}$, which is trivial if and only if $\ker(\mathbf J_{f_{\bar\theta},\mathbf h}(\bar x))$ is trivial.
\end{proof}

\begin{theorem}\label{thm:non-branch-point}
$\bar x$ is not a branch point of $\mathcal X_{\bar\theta}$ with respect to $x\mapsto \bar Sx$.
\end{theorem}

\begin{proof}
Let $\gamma_{\mathbf h}=(\gamma_h,\gamma_z)\in\ker(\mathbf J_{f_{\bar\theta},\mathbf h}(\bar x))$. We first show that $\gamma_z=0$ using the model and normalization constraints, then $\gamma_h=0$, using the rotation constraints.

Recall from~\eqref{eq:model:lifted} and \eqref{eq:affine} that
$$
g_{\bar\theta,i}(x)=[A(\bar\theta)]_iz-\big(I_2\otimes(e_3^\top P_i)\big)y_i.
$$
Since $\bar y=0$, we have $g_{\bar\theta,i}(\bar x)=[A(\bar\theta)]_i\bar z=0$.
Note that the first column of the lifted model constraint $g_{\bar\theta,i}(x)(x^{(i)})^\top=0$ is $hg_{\bar\theta,i}(x)=0$, its partial Jacobian with respect to $\mathbf h$ at $\bar x$ is
\begin{align*}
\mathbf J_{hg_{\bar\theta,i},\mathbf h}(\bar x)
&=\underbrace{g_{\bar\theta,i}(\bar x)}_{0}\mathbf J_{h,\mathbf h}(\bar x)+\bar h\mathbf J_{g_{\bar\theta,i},\mathbf h}(\bar x)\\
&=\bar h\begin{bmatrix}\mathbf J_{g_{\bar\theta,i},h}(\bar x)&\mathbf J_{g_{\bar\theta,i},z}(\bar x)\end{bmatrix}\\
&=\bar h\begin{bmatrix}0&[A(\bar\theta)]_i\end{bmatrix}.
\end{align*}
Since $\gamma_{\mathbf h}$ lies in the kernel and $\bar h\neq0$, we have $[A(\bar\theta)]_i\gamma_z=0$ for all $i\in[N]$, i.e., $A(\bar\theta)\gamma_z=0$. By the proof of Proposition~\ref{prop:mle:unique}, $\ker(A(\bar\theta))=\mathrm{span}\{\bar z\}$, and so $\gamma_z=c\bar z$ for some scalar $c$.

The partial Jacobian of the normalization constraint $\|z\|^2-1=0$ with respect to $\mathbf h$ at $\bar x$ is $\begin{bmatrix}0&2\bar z^\top\end{bmatrix}$. Since $\gamma_{\mathbf h}$ lies in the kernel, we have 
$$
\begin{bmatrix}0&2\bar z^\top\end{bmatrix}\begin{bmatrix}\gamma_h\\\gamma_z\end{bmatrix}=2\bar z^\top\gamma_z=0.
$$
Substituting $\gamma_z=c\bar z$, we have $2c\|\bar z\|^2=2c=0$, which implies $c=0$ and hence $\gamma_z=0$.

Finally, summing~\eqref{eq:rot:col1}--\eqref{eq:rot:col3} gives $\|z_R\|^2-3h^2=0$, where $z_R$ is the (vectorized) rotation component of $z$. 
Since $\gamma_z=0$, only its partial derivative with respect to $h$ remains, which gives $-6\bar h\gamma_h=0$. Since $\bar h\neq0$, we have $\gamma_h=0$. Altogether, $\ker(\mathbf J_{f_{\bar\theta},\mathbf h}(\bar x))$ is trivial, and $\bar x$ is not a branch point by Lemma~\ref{lemma:non-branch-point}.
\end{proof}

\subsection{Restricted Slater condition}\label{appx:restricted-Slater}
Recall the dual slack variable $\bar S$ in~\eqref{eq:slack} and the constraints $f_{\bar\theta}=(f_{\bar\theta,1},\dots,f_{\bar\theta,\bar\ell})$ in \eqref{eq:constraint-f}. Denote the set of Lagrange multipliers such that $(\bar x,\lambda)$ is stationary\footnote{Since $S(\lambda)=\bar S+M(\lambda)$ and $\bar S\bar x=0$ by \eqref{eq:dual-slack-variable} and \eqref{eq:slack}, $S(\lambda)\bar x=0$ if and only if $M(\lambda)\bar x=0$.} by
\begin{align*}
\Lambda&\triangleq\{\lambda\in\mathbb{R}^{\bar\ell}\mid M(\lambda)\bar x=0\},\\ M(\lambda)&\triangleq\sum_{j=1}^{\bar\ell}\lambda_j\nabla^2f_{\bar\theta,j}(\bar x)\in \mathbb{S}^{13+24N}.
\end{align*}

\begin{definition}[{\citet[Definition~4.4]{CAP22}}]\label{def:slater}
The \textit{restricted Slater condition} holds at $(\bar x,0)$ if there exists $\lambda\in\Lambda$ such that $M(\lambda)\vert_{\ker(\bar S)\cap\bar x^\perp}\succ0$.
\end{definition}

Note that $\ker(\bar S)=\{x\mid y=0\}$ involves only the hub variable $\mathbf h=(h,z)$. Let
$$
M_\mathbf{h}(\lambda)\triangleq\sum_{j=1}^{\bar\ell}\lambda_j\nabla^2_\mathbf{hh}f_{\bar\theta,j}(\bar x)\in\mathbb{S}^{13}.
$$ 
We first simplify the condition for our problem.

\begin{lemma}\label{lemma:RS}
The restricted Slater condition holds at $(\bar x,0)$ if there exists $\lambda\in\Lambda$ such that $M_{\mathbf h}(\lambda)|_{\bar{\mathbf h}^\perp}\succ0$.
\end{lemma}

\begin{proof}
Let $\gamma=(\gamma_{\mathbf h},\gamma_y)\in\ker(\bar S)\cap\bar x^\perp$ be nonzero. Since $\ker(\bar S)=\{x\mid y=0\}$, we have $\gamma_y=0$, and hence $\gamma_{\mathbf h}\neq0$. Since $\bar x=(\bar{\mathbf h},0)$, $\gamma\in\bar x^\perp$ gives $\bar{\mathbf h}^\top\gamma_{\mathbf h}=0$, i.e., $\gamma_{\mathbf h}\in\bar{\mathbf h}^\perp$. Moreover, since $\gamma^\top M(\lambda)\gamma=\gamma_{\mathbf h}^\top M_{\mathbf h}(\lambda)\gamma_{\mathbf h}$, we have $M_{\mathbf h}(\lambda)|_{\bar{\mathbf h}^\perp}\succ0$ implies $M(\lambda)|_{\ker(\bar S)\cap\bar x^\perp}\succ0$.
\end{proof}

Although Lemma~\ref{lemma:RS} reduces the condition for $\lambda$ such that the restricted Slater condition holds, we still don't know how to find it. Consider the linear space
$$
\mathcal V\triangleq\{V\in\mathbb{S}^{13}\mid V\bar{\mathbf h}=0\}.
$$
We next show that if $M_{\mathbf h}(\Lambda)=\mathcal V$, then we can take $P\in\mathcal V$ the projection matrix onto $\bar{\mathbf h}^\perp$ such that $M_{\mathbf h}(\lambda)=P$ for some $\lambda\in\Lambda$ and that the restricted Slater condition holds.

\begin{theorem}\label{thm:restricted-Slater}
$M_\mathbf h(\Lambda)=\mathcal V$ and the restricted Slater condition holds.
\end{theorem}

Note that $M_\mathbf h(\Lambda)\subseteq\mathcal V$ is trivial by construction, it suffices to show $\mathcal V\subseteq M_\mathbf h(\Lambda)$. Let 
\begin{align}
W&\triangleq(0,\bar z^\perp)\subseteq\mathbb{R}\times\mathbb{R}^{12}, \label{eq:RS:W}\\
\eta&\triangleq(\|\bar z\|^2,-\bar h\bar z)\in\mathbb{R}\times\mathbb{R}^{12}\label{eq:RS:eta}.
\end{align}
Since both $\bar{\mathbf h}=(\bar h, \bar z)$ and $\eta$ lie in $\mathbb R\times \text{span}\{\bar z\}$, and are orthogonal to each other, they form a basis for this 2-dimensional subspace. Therefore, we have 
\begin{equation}\label{eq:RS:decompose}
\text{span}\{\bar{\mathbf h}\}\oplus\text{span}\{\eta\}\oplus W=\mathbb{R}^{13}.
\end{equation}
Note that the three subspaces are orthogonal to each other.
Before proving Theorem~\ref{thm:restricted-Slater}, we first establish the following two Lemmas.

\begin{lemma}\label{lemma:RS:model}
$\mathrm{sym}(\mu\beta^\top)\in M_{\mathbf h}(\Lambda)$ for all $\mu\in\bar{\mathbf h}^\perp$ and $\beta\in W$.
\end{lemma}

\begin{proof}
We first show that $W$ in \eqref{eq:RS:W} is spanned by the rows of the partial Jacobian of the model constraints $g_{\theta,i}(x)=0$ in~\eqref{eq:model:lifted} with respect to $\mathbf h$. Denote by $g_{\bar\theta,i,k}$ the $k$-th entry of $g_{\bar\theta,i}$ for $k\in\{1,2\}$. Its partial Jacobian with respect to $\mathbf h$ at $\bar x$ is $\beta_{i,k}^\top$, where
$$
\beta_{i,k}\triangleq(0,[A(\bar\theta)]_{i,k}^\top)\in\mathbb R\times\mathbb R^{12}
$$
and $[A(\bar\theta)]_{i,k}$ is the $k$-th row of $[A(\bar\theta)]_i$. Since $A(\bar\theta)\bar z=0$, every row of $A(\bar\theta)$ is orthogonal to $\bar z$, and so $\beta_{i,k}\in W$. Moreover, the rows of $A(\bar\theta)$ span a subspace of dimension $\mathrm{rank}(A(\bar\theta))=11=\dim(\bar z^\perp)$ by the proof of Proposition~\ref{prop:mle:unique}, hence $W=\mathrm{span}\{\beta_{i,k}\}$.

We next construct a multiplier in $\Lambda$. Fix $i\in[N]$, $k\in\{1,2\}$, and let $\mu\in\bar{\mathbf h}^\perp$. Denote by $\mathbf h_a$ the $a$-th entry of $\mathbf h$ for $a\in[13]$. Since $\mathbf h$ is part of $x^{(i)}$, the lifted model constraint $g_{\bar\theta,i}(x)(x^{(i)})^\top=0$ contains $\mathbf h_ag_{\bar\theta,i,k}(x)=0$ for all $a\in[13]$. Let $\lambda'\in\mathbb R^{\bar\ell}$ assign $\mu_a$ to $\mathbf h_ag_{\bar\theta,i,k}$ for each $a\in[13]$, and $0$ to all the other constraints. Since each $\mathbf h_ag_{\bar\theta,i,k}$ is a homogeneous quadratic, its Hessian satisfies $\nabla^2(\mathbf h_ag_{\bar\theta,i,k})(\bar x)\,\bar x=\nabla(\mathbf h_ag_{\bar\theta,i,k})(\bar x)$, and so
\begin{align*}
M(\lambda')\bar x
&=\sum_{a=1}^{13}\mu_a\nabla(\mathbf h_ag_{\bar\theta,i,k})(\bar x)\\
&=\sum_{a=1}^{13}\mu_a\Big(\underbrace{g_{\bar\theta,i,k}(\bar x)}_{0}\nabla\mathbf h_a(\bar x)+\bar{\mathbf h}_a\nabla g_{\bar\theta,i,k}(\bar x)\Big)\\
&=\underbrace{(\mu^\top\bar{\mathbf h})}_{0}\nabla g_{\bar\theta,i,k}(\bar x)=0,
\end{align*}
hence $\lambda'\in\Lambda$.

Finally, we compute $M_{\mathbf h}(\lambda')$. Since $g_{\bar\theta,i,k}(x)=\beta_{i,k}^\top\mathbf h-(e_3^\top P_i)y_{i,k}$, where $y_{i,k}$ is the $k$-th block of $y_i$, we have  $\mathbf h_ag_{\bar\theta,i,k}(x)=(e_a^\top\mathbf h)(\beta_{i,k}^\top\mathbf h)-(e_a^\top\mathbf h)(e_3^\top P_i)y_{i,k}$.
Its partial Hessian with respect to $\mathbf h$ is $e_a\beta_{i,k}^\top+\beta_{i,k}e_a^\top=2\,\mathrm{sym}(e_a\beta_{i,k}^\top)$, and since $\sum_{a=1}^{13}\mu_ae_a=\mu$,
$$
M_{\mathbf h}(\lambda')=2\sum_{a=1}^{13}\mu_a\,\mathrm{sym}(e_a\beta_{i,k}^\top)=2\,\mathrm{sym}(\mu\beta_{i,k}^\top).
$$
Since $\Lambda$ and $M_{\mathbf h}(\cdot)$ are linear, $\mathrm{sym}(\mu\beta_{i,k}^\top)\in M_{\mathbf h}(\Lambda)$ for all $i\in[N]$ and $k\in\{1,2\}$. 
Altogether, since $W=\mathrm{span}\{\beta_{i,k}\}$, $\text{sym}(\mu\beta^\top)\in M_{\mathbf h}(\Lambda)$ for every $\beta\in W$.
\end{proof}

\begin{lemma}\label{lemma:RS:rotation}
Let $\eta$ be defined in~\eqref{eq:RS:eta}.
Then there exists $\alpha\in\Lambda$ such that $\eta^\top M_{\mathbf h}(\alpha)\eta\neq0$
\end{lemma}

\begin{proof}
The goal is to find a multiplier $\alpha$ such that $\eta^\top M_{\mathbf h}(\alpha)\eta\neq0$.
Recall from Appendix~\ref{appx:rot} the rotation constraints $r_j(\mathbf h)=0$ for $j\in[20]$, where $\mathbf h=(h,z)$. Let $\alpha\in\mathbb R^{\bar\ell}$ assign $\alpha_j$ to $r_j$ for $j\in[20]$ and $0$ to all the other constraints, and let
$$
q_\alpha(\mathbf h)\triangleq\sum_{j=1}^{20}\alpha_jr_j(\mathbf h).
$$
Since $q_\alpha$ involves only $\mathbf h$ and $\bar x=(\bar{\mathbf h},0)$, we have $\alpha\in\Lambda$ if and only if $(\nabla^2q_\alpha)\bar{\mathbf h}=0$. Moreover, since each $r_j$ is a homogeneous quadratic, $(\nabla^2q_\alpha)\bar{\mathbf h}=\nabla q_\alpha(\bar{\mathbf h})$ and $\eta^\top M_{\mathbf h}(\alpha)\eta=\eta^\top(\nabla^2q_\alpha)\eta=2q_\alpha(\eta)$. Therefore, it suffices to find $\alpha$ such that (\romannumeral1) $\nabla q_\alpha(\bar{\mathbf h})=0$ and (\romannumeral2) $q_\alpha(\eta)\neq0$.

We choose $\alpha$ using the column constraints~\eqref{eq:rot:col1}--\eqref{eq:rot:col3} and the determinant constraints~\eqref{eq:rot:det1}--\eqref{eq:rot:det3}. Recall from Appendix~\ref{appx:rot} that $c^{(i)}$ is the $i$-th column of $R$, and denote by $\bar c^{(i)}$ and $c^{(i)}_\mathrm{gt}$ the $i$-th columns of $\bar R$ and $R_\mathrm{gt}$, the rotation components of $\bar z$ and $z_\mathrm{gt}$, respectively. Then~\eqref{eq:rot:col1}--\eqref{eq:rot:col3} read $\|c^{(i)}\|^2-h^2=0$, and~\eqref{eq:rot:det1}--\eqref{eq:rot:det3} read $c^{(j)}\times c^{(k)}-hc^{(i)}=0$ for $(i,j,k)\in\text{cyclic}((1,2,3))$. Let $\alpha$ assign $-\tfrac12$ to each $\|c^{(i)}\|^2-h^2$, the entries of $c^{(i)}_\mathrm{gt}\in\mathbb R^3$ to the three scalar constraints in $c^{(j)}\times c^{(k)}-hc^{(i)}$, and $0$ to the remaining rotation constraints. Then
\begin{equation}\label{eq:RS:constructed-rot}
\begin{aligned}
q_\alpha(\mathbf h)&=\sum_{(i,j,k)}(c^{(i)}_\text{gt})^\top(c^{(j)}\times c^{(k)}-hc^{(i)})\\
&\quad -\frac12(\|c^{(i)}\|^2-h^2),
\end{aligned}
\end{equation}
where the sum is over $(i,j,k)\in\text{cyclic}((1,2,3))$.

For (\romannumeral1), since $q_\alpha$ does not involve $t$, it suffices to show that its partial derivatives with respect to $h$ and each $c^{(i)}$ vanish at $\bar{\mathbf h}$. Since $\bar z=z_\mathrm{gt}/\|z_\mathrm{gt}\|$ and $\bar h=1/\|z_\mathrm{gt}\|$, we have $\bar c^{(i)}=\bar h\,c^{(i)}_\mathrm{gt}$. By~\eqref{eq:RS:constructed-rot},
$$
\nabla_hq_\alpha(\mathbf h)=-\sum_{i=1}^3(c^{(i)}_\mathrm{gt})^\top c^{(i)}+3h,
$$
and hence
\begin{equation}\label{eq:RS:rot-multiplier}
\nabla_hq_\alpha(\bar{\mathbf h})=-\bar h\sum_{i=1}^3\|c^{(i)}_\mathrm{gt}\|^2+3\bar h=-3\bar h+3\bar h=0.
\end{equation}
Note that $c^{(i)}$ appears in exactly four terms of~\eqref{eq:RS:constructed-rot}:
\begin{align*}
-h(c^{(i)}_\text{gt})^\top c^{(i)}, \quad
-\frac12\|c^{(i)}\|^2, \quad \\
(c^{(j)}_\text{gt})^\top(c^{(k)}\times c^{(i)}),\quad 
(c^{(k)}_\text{gt})^\top(c^{(i)}\times c^{(j)}).
\end{align*}
Recall the cross product property $a^\top(b\times c)=b^\top(c\times a)=c^\top(a\times b)$, the corresponding derivatives with respect to $c^{(i)}$ are
$$
-hc^{(i)}_\text{gt}, \quad
-c^{(i)}, \quad 
c^{(j)}_\text{gt}\times c^{(k)}, \quad
c^{(j)}\times c^{(k)}_\text{gt}.
$$
Since $\bar c^{(i)}=\bar h c_\text{gt}^{(i)}$ and $c^{(j)}_\text{gt}\times c^{(k)}_\text{gt}=c^{(i)}_\text{gt}$, we have 
$$
\nabla_{c^{(i)}}q_\alpha(\bar{\mathbf h})=-\bar hc^{(i)}_\text{gt}-\bar hc^{(i)}_\text{gt}+\bar hc^{(i)}_\text{gt}+\bar hc^{(i)}_\text{gt}=0.
$$ 
Together with \eqref{eq:RS:rot-multiplier}, $\nabla q_\alpha(\bar{\mathbf h})=0$, i.e., $\alpha\in\Lambda$.

For (\romannumeral2), since $\eta=(\|\bar z\|^2,-\bar h\bar z)$ and with $\bar c^{(i)}=\bar h c_\text{gt}^{(i)}$ and $\|\bar z\|^2=1$, we have
\begin{align*}
q_\alpha(\eta)&=\sum_{(i,j,k)}(c^{(i)}_\text{gt})^\top(\bar h\bar c^{(j)}\times\bar h\bar c^{(k)}+\|\bar z\|^2\bar h\bar c^{(i)})\\
&\quad -\frac12(\bar h^2\|\bar c^{(i)}\|^2-\|\bar z\|^4)\\
&=\sum_{(i,j,k)}(c^{(i)}_\text{gt})^\top(\bar h^2c^{(j)}_\text{gt}\times\bar h^2c^{(k)}_\text{gt}+\bar h^2c^{(i)}_\text{gt})\\
&\quad -\frac12(\bar h^4-1)\\
&=\sum_{(i,j,k)}\bar h^4(c^{(i)}_\text{gt})^\top(c^{(j)}_\text{gt}\times c^{(k)}_\text{gt})+\bar h^2(c^{(i)}_\text{gt})^\top c^{(i)}_\text{gt}\\
&\quad -\frac12(\bar h^4-1)\\
&=3(\bar h^4+\bar h^2-\frac12\bar h^4+\frac12)=\frac32(\bar h^2+1)^2>0.
\end{align*}
Hence $q_\alpha(\eta)\neq0$, i.e., $\eta^\top M_{\mathbf h}(\alpha)\eta\neq0$.
\end{proof}

Lemma~\ref{lemma:RS:model} constructs a family of matrices in $M_{\mathbf h}(\Lambda)$ by the lifted model constraints. Lemma~\ref{lemma:RS:rotation} supplies another matrix by the rotation constraints. We next use these to prove Theorem~\ref{thm:restricted-Slater} that together, they span the linear space $\mathcal V$.

\begin{proof}
Let $\lambda\in\Lambda$. Since $M(\lambda)\bar x=0$, where $\bar x=(\bar{\mathbf h}, 0)$, we have $M_{\mathbf h}(\lambda)\bar{\mathbf h}=0$. Hence $M_{\mathbf h}(\lambda)\in\mathcal V$, and so $M_{\mathbf h}(\Lambda)\subseteq\mathcal V$. We proceed to show $\mathcal V\subseteq M_{\mathbf h}(\Lambda)$. 

Let $V\in\mathcal V$. Next, let $\alpha\in\Lambda$ satisfy Lemma~\ref{lemma:RS:rotation} and
$$
V'\triangleq V-c\,M_{\mathbf h}(\alpha),\qquad c\triangleq\frac{\eta^\top V\eta}{\eta^\top M_{\mathbf h}(\alpha)\eta},
$$
which is well defined since $\eta^\top M_{\mathbf h}(\alpha)\eta\neq0$. Since $V\in\mathcal V$ and 
\begin{equation}\label{eq:RS:Mhalpha}
M_{\mathbf h}(\alpha)\in M_{\mathbf h}(\Lambda)\subseteq\mathcal V,
\end{equation}
$V'$ also lies in $\mathcal V$. By construction,
$$
\eta^\top V'\eta=\eta^\top V\eta-\frac{\eta^\top V\eta}{\eta^\top M_{\mathbf h}(\alpha)\eta}(\eta^\top M_{\mathbf h}(\alpha)\eta)=0.
$$
Let $\hat\eta\triangleq\eta/\|\eta\|$ and $\hat\beta\triangleq V'\hat\eta$. Since $V'$ is symmetric with $V'\bar{\mathbf h}=0$ and $\eta^\top V'\eta=0$, we have
\begin{align*}
\hat\beta^\top\bar{\mathbf h}&=\hat\eta^\top V'\bar{\mathbf h}=0,\\
\hat\beta^\top\hat\eta&=\hat\eta^\top V'\hat\eta=0.
\end{align*}
By~\eqref{eq:RS:decompose}, $\hat\beta\perp\bar{\mathbf h}$ and $\hat\beta\perp\eta$ imply $\hat\beta\in W$. Moreover, $\hat\eta\in\bar{\mathbf h}^\perp$ since $\eta\perp\bar{\mathbf h}$. Hence by Lemma~\ref{lemma:RS:model}, we have
\begin{equation}\label{eq:RS:V1}
V_1\triangleq2\,\mathrm{sym}(\hat\eta\hat\beta^\top)=\hat\eta\hat\beta^\top+\hat\beta\hat\eta^\top\in M_{\mathbf h}(\Lambda).
\end{equation}
Let $V_2\triangleq V'-V_1$. Then by $\|\hat\eta\|=1$, 
\begin{align*}
V_2\bar{\mathbf h}&=V'\bar{\mathbf h}-\hat\eta(\hat\beta^\top\bar{\mathbf h})-\hat\beta(\hat\eta^\top\bar{\mathbf h})=0,\\
V_2\hat\eta&=V'\hat\eta-\hat\eta(\hat\beta^\top\hat\eta)-\hat\beta(\hat\eta^\top\hat\eta)=\hat\beta-0-\hat\beta=0.
\end{align*}
Since $V_2$ is symmetric and $\mathrm{span}\{\bar{\mathbf h},\eta\}\subseteq\ker(V_2)$, we have $\mathrm{range}(V_2)=\ker(V_2)^\perp\subseteq\mathrm{span}\{\bar{\mathbf h},\eta\}^\perp=W$ by~\eqref{eq:RS:decompose}. By the spectral theorem, write $V_2=\sum_k\sigma_kw_kw_k^\top$ with $\sigma_k\neq0$, where the eigenvectors $w_k$ lie in $\mathrm{range}(V_2)\subseteq W\subseteq\bar{\mathbf h}^\perp$. Taking $\mu=\beta=w_k$ in Lemma~\ref{lemma:RS:model} gives $w_kw_k^\top=\mathrm{sym}(w_kw_k^\top)\in M_{\mathbf h}(\Lambda)$, and hence $V_2\in M_{\mathbf h}(\Lambda)$. Together with $M_\mathbf h(\alpha)\in  M_{\mathbf h}(\Lambda)$ by \eqref{eq:RS:Mhalpha} and $V_1\in M_{\mathbf h}(\Lambda)$ by \eqref{eq:RS:V1}, 
$$
V=c\,M_{\mathbf h}(\alpha)+V_1+V_2\in M_{\mathbf h}(\Lambda).
$$ 
That is, $\mathcal V\subseteq M_{\mathbf h}(\Lambda)$ and so $M_{\mathbf h}(\Lambda)=\mathcal V$.

Finally, let $P\in \mathcal V$ be the projection matrix onto $\bar{\mathbf h}^\perp$. Since $\omega^\top P\omega=\|\omega\|^2>0$ for all $\omega\in\bar{\mathbf h}^\perp$ with $\omega\neq0$, we have $P\vert_{\bar{\mathbf h}^\perp}\succ0$. Since $P\in\mathcal V=M_{\mathbf h}(\Lambda)$, there exists some $\lambda\in\Lambda$ such that $M_\mathbf h(\lambda)=P$. By Lemma~\ref{lemma:RS}, the restricted Slater condition holds.
\end{proof}

\subsection{Proof of Corollary~\ref{cor:stable}}\label{appx:cor:stable}
\begin{proof}
Let $p^*_{\mathrm{QCQP}}$ and $p^*_{\mathrm{SDP}}$ be the optimal values of Problem~\ref{prob:lifted} and \eqref{eq:sdp}, respectively.
By Theorem~\ref{thm:stable}, \eqref{eq:sdp} is exact, i.e., it is tight and has a unique rank 1 solution $x^*(x^*)^\top$ with $x^*=(h^*,z^*,y^*)$. Since $x^*$ is feasible for Problem~\ref{prob:lifted} with objective value $p^*_{\mathrm{SDP}}$, and $p^*_{\mathrm{SDP}}\leq p^*_{\mathrm{QCQP}}$ because \eqref{eq:sdp} is a relaxation of Problem~\ref{prob:lifted}, it is a global minimizer. Conversely, any global minimizer $x$ of Problem~\ref{prob:lifted} gives an optimal solution $xx^\top$ of \eqref{eq:sdp} by tightness, hence $xx^\top=x^*(x^*)^\top$, i.e., $x\in[x^*]$. By Lemma~\ref{lemma:feasible-set:lifted:restricted}, $h^*\neq0$, and so by the proofs of Theorem~\ref{thm:naive} and~\ref{thm:lifted}, $z^*/h^*$ is the unique global minimizer of Problem~\ref{prob:mle}.
\end{proof}

\section{SDP reformulations}\label{secC}
\setcounter{equation}{0}

\begin{lemma}[Grone's Theorem~{\citep[Theorem 10.1]{VA15}}]\label{lemma:grone}
Let $X\in\mathbb{S}^n$ whose corresponding variable dependency graph $\mathcal{G}$ is a chordal graph. 
Denote $X^{(i)}$ the principal submatrices of $X$ where their corresponding variable dependency graphs are the maximal cliques of $\mathcal{G}$. Then $X$ is positive semidefinite completable if and only if $X^{(i)}\succeq0$ for all $i$.
\end{lemma}

\subsection{Proof of Theorem~\ref{thm:chordal}}\label{appx:proof:chordal}

\begin{proof}
Denote by $\mathcal F$ and $\mathcal F_\mathrm{c}$ the feasible sets of~\eqref{eq:sdp} and~\eqref{eq:chordal:sdp}, and consider the mapping
$$
\phi:\mathcal F\to\mathcal F_\mathrm{c},\quad X\mapsto(X^{(1)},\dots,X^{(N)})
$$
that sends $X$ to its clique blocks. We first show that $\phi$ is well-defined and preserves the objective value. Let $X\in\mathcal F$. Since $X\succeq0$, its principal submatrices satisfy $X^{(i)}\succeq0$. Since $\tilde C_{g_{\theta,i},k}$ and $\tilde C_{r_j}$ contain all the nonzero entries of $C_{g_{\theta,i},k}$ and $C_{r_j}$,
\begin{align*}
&\tilde C_{g_{\theta,i},k}\bullet X^{(i)}=C_{g_{\theta,i},k}\bullet X=0,\\
&\tilde C_{r_j}\bullet X^{(1)}_{\mathbf{hh}}=C_{r_j}\bullet X=0.
\end{align*}
The remaining constraints of~\eqref{eq:chordal:sdp} hold since $X^{(1)}_{zz}=X_{zz}$, $X^{(i)}_{u_iu_i}=X_{u_iu_i}$, and $X^{(i)}_{\mathbf{hh}}=X_{\mathbf{hh}}$ for all $i\in[N]$, hence $\phi(X)\in\mathcal F_\mathrm{c}$. Since $\mathrm{trace}(X_{yy})=\sum_{i=1}^N\mathrm{trace}(X^{(i)}_{y_iy_i})$, $\phi$ preserves the objective value.

We next show that $\phi$ is bijective between optimal sets. It is injective because the optimal sets of \eqref{eq:sdp} is a singleton by Theorem~\ref{thm:stable}. For surjectivity, let $(X^{(1)},\dots,X^{(N)})\in\mathcal F_\mathrm{c}$ be an optimal solution of \eqref{eq:chordal:sdp}. Then by Lemma~\ref{lemma:grone}, it has a completion $X\in\mathcal F$ with the same objective value, hence optimal of \eqref{eq:sdp}. 
Therefore, $\phi$ is a bijection between the optimal sets, and \eqref{eq:chordal:sdp} is equivalent to~\eqref{eq:sdp} by Definition~\ref{def:prob-equivalence}.
\end{proof}

\subsection{Proof of Theorem~\ref{thm:qrr}}\label{appx:proof:qrr}

\begin{proof}
Denote by $\mathcal F_\mathrm r$ and $\mathcal F_\mathrm c$ the feasible sets of~\eqref{eq:reduced:sdp} and~\eqref{eq:chordal:sdp}, and consider the mapping
\begin{align*}
&\phi:\mathcal F_\mathrm r\to\mathcal F_\mathrm c,\\
&(\check X^{(1)},\dots,\check X^{(N)})\mapsto(\Xi_1\check X^{(1)}\Xi_1^\top,\dots,\Xi_N\check X^{(N)}\Xi_N^\top).
\end{align*}

We first show that $\phi$ is well-defined and preserves the objective value. Let $(\check X^{(1)},\dots,\check X^{(N)})\in\mathcal F_\mathrm r$. $\Xi_i\check X^{(i)}\Xi_i^\top\succeq0$ since $\check X^{(i)}\succeq0$ for all $i\in[N]$. Since~\eqref{eq:reduced:sdp} is obtained by substituting~\eqref{eq:qrr:param} into~\eqref{eq:chordal:sdp}, the objective value remains the same and each constraint holds if and only if the corresponding constraint of \eqref{eq:chordal:sdp} holds at $\Xi_i\check X^{(i)}\Xi_i^\top$. 

We next show that $\phi$ is bijective between optimal sets. For injectivity, let $\Xi_i\check X^{(i)}\Xi_i^\top=\Xi_i\check X'^{(i)}\Xi_i^\top$, then multiplying by $\Xi_i^\top$ on the left and $\Xi_i$ on the right, and canceling $\Xi_i^\top\Xi_i$, which is invertible since $\Xi_i$ has full column rank, gives $\check X^{(i)}=\check X'^{(i)}$. For surjectivity, let $(X^{(1)},\dots,X^{(N)})$ be an optimal solution of \eqref{eq:chordal:sdp}. 
By Theorem~\ref{thm:stable} and~\ref{thm:chordal}, $X^{(i)}$ are the clique blocks of the unique rank 1 optimal solution of~\eqref{eq:sdp}, which are also rank 1. That is, $X^{(i)}=x^{(i)}(x^{(i)})^\top$, and since $x^{(i)}$ also corresponds to the subvariable of the optimal solution of Problem~\ref{prob:lifted}, we have $Q_ix^{(i)}=0$ by~\eqref{eq:qrr:Q}. Thereby, $x^{(i)}=\Xi_i\check x^{(i)}$ for some $\check x^{(i)}\in\mathbb R^{34}$. 
Then $(\check x^{(1)}(\check x^{(1)})^\top,\dots,\check x^{(N)}(\check x^{(N)})^\top)$ is feasible for~\eqref{eq:reduced:sdp}, maps to $(X^{(1)},\dots,X^{(N)})$ with the same objective value by the first paragraph, hence is optimal. Therefore, $\phi$ is a bijection between the optimal sets, and~\eqref{eq:reduced:sdp} is equivalent to~\eqref{eq:chordal:sdp} by Definition~\ref{def:prob-equivalence}.
\end{proof}

\subsection{Proof of Proposition~\ref{prop:hub:rank}}\label{appx:proof:hub:rank}

\begin{proof}
Suppose that $\check X^{(i)}$ has rank 1 for all $i\in[N]$. $\check X^{(1)}_{\mathbf{hh}}$ is a principal submatrix of $\check X^{(1)}$ and is nonzero since $\mathrm{trace}(\check X^{(1)}_{zz})=1$, hence is also rank 1.

Conversely, suppose that $\check X^{(1)}_{\mathbf{hh}}$ has rank 1, then $\check X^{(i)}_{\mathbf{hh}}$ is also  rank 1 since $\check X^{(i)}_{\mathbf{hh}}=\check X^{(1)}_{\mathbf{hh}}$ for all $i\in[N]$. 
Fix $i\in[N]$ and let $X^{(i)}\triangleq\Xi_i\check X^{(i)}\Xi_i^\top$ where $\Xi_i$ is defined in Section~\ref{sec:reduce:qrr}. Since $\Xi_i$ has full column rank, $X^{(i)}$ has the same rank as $\check X^{(i)}$, it suffices to show that $X^{(i)}$ has is 1.

Since $X^{(i)}_{\mathbf{hh}}=\check X^{(i)}_{\mathbf{hh}}$ by \eqref{eq:qrr:xi:hub}, $X^{(i)}_{\mathbf{hh}}$ is also rank 1. Write $X^{(i)}_{\mathbf{hh}}=\alpha_{\mathbf h}\alpha_{\mathbf h}^\top$ with $\alpha_{\mathbf h}=(\alpha_h,\alpha_z)\in\mathbb R^{13}$, where $\|\alpha_z\|^2=\mathrm{trace}(X^{(i)}_{zz})=1$. 
We claim that 
\begin{equation}\label{eq:hub:form}
X^{(i)}=\begin{pmatrix}\alpha_{\mathbf h}\alpha_{\mathbf h}^\top&\alpha_{\mathbf h}\alpha_{y_i}^\top\\\alpha_{y_i}\alpha_{\mathbf h}^\top&*\end{pmatrix}
\end{equation}
for some $\alpha_{y_i}\in\mathbb{R}^{24}$.
Let $\beta_{\mathbf h}\in\mathbb R^{13}$ with $\alpha_{\mathbf h}^\top\beta_{\mathbf h}=0$, and let $\beta\triangleq(\beta_{\mathbf h},0)\in\mathbb R^{37}$. Then
$$
\beta^\top X^{(i)}\beta=\beta_{\mathbf h}^\top X^{(i)}_{\mathbf{hh}}\beta_{\mathbf h}=(\alpha_{\mathbf h}^\top\beta_{\mathbf h})^2=0,
$$
which implies $X^{(i)}\beta=0$ since $X^{(i)}\succeq0$, and in particular $X^{(i)}_{y_i\mathbf h}\beta_{\mathbf h}=0$. That is, every row of $X^{(i)}_{y_i\mathbf h}$ is orthogonal to every vector orthogonal to $\alpha_{\mathbf h}$, and hence is a multiple of $\alpha_{\mathbf h}^\top$. Collecting these multiples into $\alpha_{y_i}\in\mathbb R^{24}$ gives $X^{(i)}_{y_i\mathbf h}=\alpha_{y_i}\alpha_{\mathbf h}^\top$, as claimed. 

We first construct a candidate\footnote{We call $M$ a candidate if it is in the form of \eqref{eq:hub:form} and equals $\Xi_i\check M\Xi_i^\top$ for some $\check M$ that is feasible for~\eqref{eq:reduced:sdp}.} for $X^{(i)}$.
Let $\alpha\triangleq(\alpha_{\mathbf h},\alpha_{y_i})\in\mathbb R^{37}$ and $Y\triangleq\alpha\alpha^\top$, which is rank 1 and is of the form~\eqref{eq:hub:form}. We first show that $Y=\Xi_i\check Y\Xi_i^\top$ for some $\check Y$ feasible for~\eqref{eq:reduced:sdp}, i.e., $Y$ is a candidate of $X^{(i)}$. Since the columns of $\Xi_i$ span $\ker(Q_i)$, we have $Q_i\Xi_i=0$ and hence 
\begin{equation}\label{eq:QandXrelation}
Q_iX^{(i)}=Q_i\Xi_i\check X^{(i)}\Xi_i^\top=0. 
\end{equation}
Note that for $j\in[13]$, the $j$-th column of~\eqref{eq:hub:form} is
$$
X^{(i)}e_j=\begin{pmatrix}\alpha_{\mathbf h}(\alpha_{\mathbf h})_j\\\alpha_{y_i}(\alpha_{\mathbf h})_j\end{pmatrix}=(\alpha_{\mathbf h})_j\,\alpha,
$$
where $e_j\in\mathbb R^{37}$ is the $j$-th standard basis vector.
Then by \eqref{eq:QandXrelation}, $Q_iX^{(i)}e_j=(\alpha_{\mathbf h})_jQ_i\alpha=0$. Since $\alpha_{\mathbf h}\neq0$, we have $Q_i\alpha=0$, that is, $\alpha$ lies in $\ker(Q_i)$ and is spanned by the columns of $\Xi_i$. Write $\alpha=\Xi_i\check\alpha$ for some $\check\alpha\in\mathbb R^{34}$, and $Y=\Xi_i\check Y\Xi_i^\top$ with $\check Y\triangleq\check\alpha\check\alpha^\top\in\mathbb S^{34}_+$.

We next check the constraints of \eqref{eq:reduced:sdp} at $\check Y$. By the proof of Theorem~\ref{thm:qrr}, it suffices to check the constraints of~\eqref{eq:chordal:sdp} at $Y$. 
The constraints on the hub block hold since $Y_{\mathbf{hh}}=X^{(i)}_{\mathbf{hh}}$. The lifted model constraints read $\tilde Q_iX^{(i)}=0$ by \eqref{eq:qrr:model}, and they hold at $Y$ since $\tilde Q_i\alpha=0$, as $\tilde Q_i$ consists of the first two rows of $Q_i$. 
For the block symmetry constraint, which requires the $12\times12$ blocks of $Y$ indexed by $z$, $y_{i,1}$, and $y_{i,2}$ to be symmetric, write $\alpha_{y_i}=(\alpha_{y_i,1},\alpha_{y_i,2})$ with $\alpha_{y_i,k}\in\mathbb R^{12}$. By~\eqref{eq:hub:form}, $X^{(i)}_{zy_{i,k}}=\alpha_z\alpha_{y_i,k}^\top$, which is symmetric by the block symmetry constraint on $X^{(i)}$. Multiplying $\alpha_z\alpha_{y_i,k}^\top=\alpha_{y_i,k}\alpha_z^\top$ by $\alpha_z$ on the right and using $\|\alpha_z\|=1$ gives $\alpha_{y_i,k}=c_k\alpha_z$ where $c_k\triangleq\alpha_{y_i,k}^\top\alpha_z\in\mathbb R$. Hence all the blocks
\begin{align*}
Y_{zz}&=\alpha_z\alpha_z^\top,\\
Y_{zy_{i,k}}&=c_k\,\alpha_z\alpha_z^\top,\\
Y_{y_{i,k}y_{i,l}}&=c_kc_l\,\alpha_z\alpha_z^\top,\quad k,l\in\{1,2\},
\end{align*}
are symmetric. Therefore, we have that $\check Y$ is feasible such that $Y$ is a candidate for $X^{(i)}$, but not necessarily optimal yet.

We next show that $Y$ has the smallest objective value among all candidates. Let $Y'\neq Y$ be a candidate for $X^{(i)}$. Since both $Y$ and $Y'$ are of the form~\eqref{eq:hub:form}, they differ only in the lower right block. We claim that $Y'_{y_iy_i}-Y_{y_iy_i}\succeq0$, i.e., $w^\top Y'_{y_iy_i}w\geq(\alpha_{y_i}^\top w)^2$ for all $w\in\mathbb R^{24}$. Let $s\triangleq-\alpha_{y_i}^\top w$ and $\tilde w\triangleq(s\alpha_{\mathbf h}/\|\alpha_{\mathbf h}\|^2,w)\in\mathbb R^{37}$. Since $Y'\succeq0$ and is of the form~\eqref{eq:hub:form}, we have
\begin{align*}
0\leq\tilde w^\top Y'\tilde w
&=s^2+2s(\alpha_{y_i}^\top w)+w^\top Y'_{y_iy_i}w\\
&=(\alpha_{y_i}^\top w)^2-2(\alpha_{y_i}^\top w)^2+w^\top Y'_{y_iy_i}w\\
&=w^\top(Y'_{y_iy_i}-Y_{y_iy_i})w,
\end{align*}
as claimed. Since $Y'\neq Y$, the matrix $Y'_{y_iy_i}-Y_{y_iy_i}\succeq0$ is nonzero, and so $\mathrm{trace}(Y'_{y_iy_i})>\mathrm{trace}(Y_{y_iy_i})$. That is, every candidate other than $Y$ has a strictly larger objective value, hence $Y$ is the only optimal candidate. Altogether, $X^{(i)}=Y=\alpha\alpha^\top$ is rank 1.
\end{proof}

\subsection{Proof of Proposition~\ref{prop:recover}}\label{appx:proof:recover}
\begin{proof}
Let $p_3^*$ and $p_\text{r}^*$ be the optimal values of Problem~\ref{prob:lifted} and \eqref{eq:reduced:sdp}, respectively. 
By Proposition~\ref{prop:hub:rank}, $\check X^{(i)}$ has rank 1 for all $i\in[N]$. Since $\Xi_i$ has full column rank and preserves the hub block, $X^{(i)}\triangleq\Xi_i\check X^{(i)}\Xi_i^\top=x^{(i)}(x^{(i)})^\top$ for some $x^{(i)}=(h,z,y_i)$, where the sign of each $x^{(i)}$ is chosen such that its hub part is $(h,z)$. Let $x\triangleq(h,z,y)$ with $y=(y_1,\dots,y_N)$. By the well-definedness part of the proof of Theorem~\ref{thm:qrr}, $(X^{(1)},\dots,X^{(N)})$ is feasible for \eqref{eq:chordal:sdp}, and since $X^{(i)}=x^{(i)}(x^{(i)})^\top$, its constraints read the constraints of Problem~\ref{prob:lifted} at $x$. Hence $x$ is feasible for Problem~\ref{prob:lifted} with objective value $p^*_{\mathrm r}=\|y\|^2\geq p^*_3$. Since~\eqref{eq:reduced:sdp} is a relaxation of Problem~\ref{prob:lifted}, $p_\text{r}^*\leq p_3^*$. Altogether, $p_3^*=p_\text{r}^*$, and hence $x$ is a global minimizer of Problem~\ref{prob:lifted}.

For the second part, write $y_i=v_i\otimes z$ as in the proof of Lemma~\ref{lemma:feasible-set:lifted:restricted}, such that $(h,z,v)\in\mathcal W'_\theta$ in \eqref{eq:W-normalized}. Since $e_3^\top P_iz\neq0$ for all $i\in[N]$ and $h\neq0$, the argument in the proof of Theorem~\ref{thm:ACQ} with $\bar\theta$ replaced by $\theta$ gives $z/h\in\mathcal Z$ with objective value $\|v\|^2=\|y\|^2=p^*_3$ for Problem~\ref{prob:mle}. 
Since the mappings in the proofs of Theorem~\ref{thm:naive} and~\ref{thm:lifted} are well-defined and preserve the objective value for any $\theta$, Problem~\ref{prob:lifted} is a relaxation of Problem~\ref{prob:mle}, and hence $z/h$ is a global minimizer of Problem~\ref{prob:mle}.
\end{proof}

\end{appendices}

\bibliography{sn-bibliography}

\end{document}